%% file: aaai_main.tex
\documentclass[letterpaper]{article} 
\usepackage{aaai2026}  
\usepackage{times}  
\usepackage{helvet}  
\usepackage{courier}  
\usepackage[hyphens]{url}  
\usepackage{graphicx} 
\usepackage{natbib}  
\usepackage{caption} 
\usepackage{algorithm}
\usepackage{algorithmic}
\usepackage{amsmath}
\usepackage{bm}
\usepackage{amsthm}
\usepackage{booktabs}       
\usepackage{amsfonts}       
\usepackage{nicefrac}       
\usepackage{microtype}      
\usepackage{xcolor}         
\usepackage{mathtools}

\newtheorem{proposition}{Proposition}
\newtheorem*{proposition*}{Proposition}
\usepackage{siunitx}  
\usepackage{multirow}
\usepackage{thmtools}

\usepackage{cleveref}
\usepackage{amsmath}
\usepackage{autonum}
\usepackage{xcolor}

\usepackage{newfloat}
\usepackage{listings}

\DeclareCaptionStyle{ruled}{labelfont=normalfont,labelsep=colon,strut=off} 
\floatstyle{ruled}
\newfloat{listing}{tb}{lst}{}
\floatname{listing}{Listing}
\title{Tree-Based Stochastic Optimization for Solving Large-Scale Urban Network Security Games}
\author{
    Shuxin Zhuang\textsuperscript{\rm 1,2}\equalcontrib,
    Linjian Meng\textsuperscript{\rm 3}\equalcontrib,
    Shuxin Li\textsuperscript{\rm 4},
    Minming Li\textsuperscript{\rm 1}\footnotemark[2],
    Youzhi Zhang\textsuperscript{\rm 2}\thanks{Corresponding authors.}
}
\affiliations{
    \textsuperscript{\rm 1}Department of Computer Science, City University of Hong Kong, HKSAR, China\\
    \textsuperscript{\rm 2}CAIR, Hong Kong Institute of Science \& Innovation, CAS, HKSAR, China\\
    \textsuperscript{\rm 3}National Key Laboratory for Novel Software Technology, Nanjing University, Nanjing, China\\
    \textsuperscript{\rm 4}College of Computing and Data Science, Nanyang Technological University, Singapore\\
    shuxin.zhuang@my.cityu.edu.hk; menglinjian@smail.nju.edu.cn; shuxin.li@ntu.edu.sg; \\
    minming.li@cityu.edu.hk; youzhi.zhang@cair.cas.org.hk
}

\usepackage{bibentry}

\begin{document}

\maketitle

\begin{abstract}

Urban Network Security Games (UNSGs), which model the strategic allocation of limited security resources on city road networks, are critical for urban safety. However, finding a Nash Equilibrium (NE) in large-scale UNSGs is challenging due to their massive and combinatorial action spaces. One common approach to addressing these games is the Policy-Space Response Oracle (PSRO) framework, which requires computing best responses (BR) at each iteration. However, precisely computing exact BRs is impractical in large-scale games, and employing reinforcement learning to approximate BRs inevitably introduces errors, which limits the overall effectiveness of the PSRO methods. Recent advancements in leveraging non-convex stochastic optimization to approximate an NE offer a promising alternative to the burdensome BR computation. However, utilizing existing stochastic optimization techniques with an unbiased loss function for UNSGs remains challenging because the action spaces are too vast to be effectively represented by neural networks. To address these issues, we introduce \textbf{T}ree-based \textbf{S}tochastic \textbf{O}ptimization (TSO), a framework that bridges the gap between the stochastic optimization paradigm for NE-finding and the demands of UNSGs. Specifically, we employ the tree-based action representation that maps the whole action space onto a tree structure, addressing the challenge faced by neural networks in representing actions when the action space cannot be enumerated. We then incorporate this representation into the loss function and theoretically demonstrate its equivalence to the unbiased loss function. To further enhance the quality of the converged solution, we introduce a sample-and-prune mechanism that reduces the risk of being trapped in suboptimal local optima. Extensive experimental results indicate the superiority of TSO over other baseline algorithms in addressing the UNSGs.
\end{abstract}



\section{Introduction}

The strategic allocation of limited security resources is crucial for public safety, with applications focused on preventing, deterring, and detecting illicit activities to enhance urban security \citep{tsai2010urban,jain2011double,Jain2013,iwashita2016simplifying,zhang2017optimal,zhang2019optimal,li2024grasper}. Many real-world security challenges, including the protection of critical infrastructure \cite{jain2011double} and the interdiction of urban criminals \cite{zhang2017optimal, zhang2019optimal, xue2021solving}, can be effectively modeled as UNSGs. In these games, a team of defender strategically deploys resources, such as police patrols or security checkpoints, across a city's road network to interdict the adversary. The adversary, in turn, aims to traverse the network from a source to a destination while minimizing the probability of being caught. However, a fundamental challenge in solving UNSGs lies in their computational complexity. As urban networks scale, the exponential growth of possible adversary paths and defender resource allocations leads to massive action spaces \cite{jain2011double}, making exact solutions computationally intractable.

A prominent line of research for solving UNSGs relies on iterative and oracle-based algorithms to approximate a Nash equilibrium (NE) \cite{nash1951non,mcmahan2003planning,jain2011double}. A foundational method in this domain is the double-oracle algorithm \cite{mcmahan2003planning,jain2011double,zhang2017optimal,zhang2019optimal}. It begins with a small, restricted set of pure strategies for each player. In each iteration, it computes an equilibrium for this restricted game and then expands the strategy sets by adding each player's best response (BR) to the opponent's current strategy. This process continues until convergence. While the classic double-oracle algorithm relies on linear programming, it loses its effectiveness in large-scale UNSGs. To overcome this limitation, PSRO \citep{lanctot2017unified} was introduced as a generalization of the double-oracle framework, embedding deep reinforcement learning (DRL) to handle immense strategy spaces, and it inherits the convergence guarantees of the double-oracle method. However, a critical bottleneck shared by these approaches is their dependence on best-response computation. Finding a best response in UNSGs is an NP-hard problem \cite{jain2011double}. Indeed, recent work \cite{xue2021solving,xue2022nsgzero,li2023solving, li2024grasper,zhuang2025solving} employs DRL to approximate the BR oracle for solving UNSGs, but these RL-based oracles often suffer from low accuracy. The resulting imprecise best responses hamper the effectiveness of PSRO framework, and lead to convergence towards suboptimal policies.

In parallel, another line of research \cite{gemp2021sample, gemp2024approximating, mengreducing}, primarily developed for normal-form games, reframes the NE-finding problem as a non-convex stochastic optimization task. In this paradigm, a loss function is constructed whose minimization corresponds to finding an NE. Player policies, often parameterized as neural networks, are then updated directly via gradient descent using batches of data from sampled gameplay, thereby bypassing the need for an explicit BR oracle. While this paradigm is promising, its direct application to large-scale UNSGs remains challenging. Its core mechanism requires representing a player's policy as an explicit probability distribution over the entire action space to facilitate action sampling. Such an enumeration-based representation becomes intractable in typical UNSGs, where the action space is vast and combinatorial.

In this paper, we propose a novel framework, Tree-based Stochastic Optimization (TSO), to solve large-scale UNSGs. Our approach bridges the gap between the stochastic optimization paradigm for NE-finding and the practical demands of large-scale security games. First, we introduce a tree-based action sampling process where each action is mapped to a unique path in a decision tree. This structure decomposes the probability of selecting an action into a product of conditional probabilities at sequential decision points, thereby enabling efficient action sampling without enumerating the full action space. We then incorporate tree-based action representation into the Nash Advantage Loss (NAL) \cite{mengreducing} and theoretically demonstrate its equivalence to the unbiased loss function. Furthermore, to enhance the quality of the converged solution, we introduce a sample-and-prune mechanism. This technique diversifies exploration by pruning the initially sampled high-probability action and forcing a re-sample from the remaining action space, thereby mitigating the risk of premature convergence to a suboptimal local optimum. Extensive experiments demonstrate that TSO significantly outperforms established baselines in UNSGs with both small and large action spaces.

\section{Preliminaries}

\subsection{Game Definition}

UNSGs involve a defender strategically placing security resources on city roads to protect against an adversary who selects a path through the city. Following the game definition from \cite{jain2011double,tsai2010urban}, we model the UNSGs on a graph $G = (V, E)$ as a one-shot and simultaneous-move game between an attacker and a team of $N$ defenders. The attacker's action, $a_\text{attacker}$, is a simple path from a start vertex $s \in V_{\text{start}}$ to a target vertex $t \in V_{\text{target}}$, with the action space $\mathcal{A}_\text{attacker}$ being the set of all such paths. Each defender $m\in\{1, \dots, N\}$ selects an edge $e_m$ from their individual action space $\mathcal{E}_m \subseteq E$. The team's action $a_\text{defender}$ is a tuple of these selected edges, and their action space is the Cartesian product $\mathcal{A}_\text{defender} = \times_{m=1}^{N} \mathcal{E}_m$.

Player utilities are determined by their joint actions. A function $U: V_{\text{target}} \to \mathbb{R}^+$ assigns a positive value $U(v_{t_i})$ to each target $v_{t_i}$, and we use target($a_\text{attacker}$) to denote the endpoint of the attacker's chosen path. Interception occurs if the attacker's path $a_\text{attacker}$ includes any edge from the defenders' action $a_\text{defender}$. A successful attacker gains utility equal to the target's value, $U(\text{target}(a_\text{attacker}))$, while an intercepted attacker incurs a penalty of the same magnitude.
Formally, the attacker's utility is defined as:
\begin{equation}
\begin{aligned}
& u_\text{attacker}(a_\text{attacker}, a_\text{defender}) \\
& \qquad = 
\begin{cases}
U(\text{target}(a_\text{attacker})) & \text{if } a_\text{attacker} \cap a_\text{defender} = \emptyset \\
-U(\text{target}(a_\text{attacker})) & \text{if } a_\text{attacker} \cap a_\text{defender} \neq \emptyset.
\end{cases}
\end{aligned}    
\end{equation}

This game can be modeled as a zero-sum normal-form game (NFG), where the defender's utility is the negative of the attacker's utility: $u_\text{defender}(a_\text{attacker}, a_\text{defender}) = -u_\text{attacker}(a_\text{attacker}, a_\text{defender})$.

A mixed strategy for a player $i \in \mathcal{P}$, where $\mathcal{P} = \{\text{attacker}, \text{defender}\}$, denoted by $\bm{x}_i$, is a probability distribution over their pure action set $\mathcal{A}_i$. The space of all valid mixed strategies for player $i$ is the probability simplex $\bm{\mathcal{X}}_i = \{ \bm{x}_i \in \mathbb{R}^{|\mathcal{A}_i|} \mid \sum_{a_i \in \mathcal{A}_i} x_i(a_i) = 1, \forall a_i, x_i(a_i) \ge 0 \}$. A strategy profile $\bm{x} = (\bm{x}_i)_{i \in \mathcal{P}}$ for the game is an element of the joint strategy space $\bm{\mathcal{X}} = \times_{i \in \mathcal{P}} \bm{\mathcal{X}}_i$. We denote the interior of this space as $\bm{\mathcal{X}}^\circ$, where every pure strategy is played with a strictly positive probability. Given a strategy profile $\bm{x}$, the expected utility for player $i$ is given by $u_i(\bm{x}) = u_i(\bm{x}_i, \bm{x}_{-i})$, where $\bm{x}_{-i}$ denotes the strategies of all players except $i$. The expected utility for player $i$ is:
\begin{equation}
\setlength\abovedisplayskip{2pt}
\setlength\belowdisplayskip{2pt}
\thinmuskip=0mu
\medmuskip=0mu
\thickmuskip=0mu
\spaceskip=-0pt
u_i(\bm{x}_i, \bm{x}_{-i}) = \sum_{\bm{a} \in \times_{i \in \mathcal{P}} \mathcal{A}_i} u_i(\bm{a}) \prod_{j \in \mathcal{P}} \bm{x}_j(a_j).
\end{equation}

Our primary objective is to compute an NE, a strategy profile $\bm{x}^*$ from which no player has a unilateral incentive to deviate. As analyzed in \cite{facchinei2003finite}, if the utility function of each player $i$ is concave over $\bm{\mathcal{X}}_i$, an NE $\bm{x}$ is such that $\langle \nabla_{x_i} u_i(\bm{x}), \bm{x}_i - \bm{x}_i' \rangle \leq 0, \ \forall i \in \mathcal{P} \text{ and } \bm{x} \in \bm{\mathcal{X}}$. For NFGs, the expected utility function $u_i(\bm{x})$ is linear with respect to a player's own strategy $\bm{x}_i$. This property allows us to quantify how close a strategy profile $\bm{x}$ is to an equilibrium using the duality gap, which measures the total exploitability of the profile:
\begin{equation}
\setlength\abovedisplayskip{2pt}
\setlength\belowdisplayskip{2pt}
\thinmuskip=0mu
\medmuskip=0mu
\thickmuskip=0mu
\spaceskip=-0pt
\mathrm{dg}(\boldsymbol{x}) = \sum_{i \in \mathcal{P}} \left( \max_{\boldsymbol{x}_i' \in \bm{\mathcal{X}}_i} u_i(\boldsymbol{x}_i', \boldsymbol{x}_{-i}) - u_i(\boldsymbol{x}_i, \boldsymbol{x}_{-i}) \right).    
\end{equation}
A strategy profile $\bm{x}$ is an NE if and only if its duality gap is zero. Thus, our goal is finding a strategy profile $\bm{x}^* \in \bm{\mathcal{X}}$ such that $\mathrm{dg}(\boldsymbol{x}^*) = 0$.

\subsection{Nash Equilibrium as Stochastic Optimization}

Many studies \citep{raghunathan2019game,goktasgenerative,marris2022turbocharging, gemp2021sample,  duan2023nash,liunfgtransformer,yongacoglu2024paths} treat computing NE as an optimization task. While various loss functions have been proposed for this purpose, many are not amenable to unbiased estimation. Recently, \citet{gemp2024approximating} and \citet{mengreducing} introduced methods based on unbiased estimation.

\citet{gemp2024approximating} leverages a property of concave games: for an interior strategy profile $\boldsymbol{x} \in \bm{\mathcal{X}}^{\circ}$, a player's utility gradients are identical across all actions if and only if $\boldsymbol{x}$ is an NE. To guarantee the existence of an interior NE, they introduce an entropy-regularized utility function for each player $i$: $u_{i}^{\tau}(\boldsymbol{x}) = u_{i}(\boldsymbol{x}) - \tau \boldsymbol{x}_{i}^{\top} \log \boldsymbol{x}_{i}$, where $\tau > 0$. Based on this insight, they proposed the first unbiased loss function $\mathcal{L}_{G}^{\tau}(\boldsymbol{x})$ for stochastic optimization.


However, as \citet{mengreducing} point out, $\mathcal{L}_{G}^{\tau}(\boldsymbol{x})$ suffers from high-variance loss estimates, which can impede convergence. 
To address this, \citet{mengreducing} introduced the surrogate Nash Advantage Loss (NAL), $\mathcal{L}_\text{NAL}^{\tau}(\bm{x})$, designed to enable low-variance, unbiased estimation using only a single random variable. The loss is defined as:
\begin{equation}
\setlength\abovedisplayskip{3pt}
\setlength\belowdisplayskip{3pt}
\thinmuskip=0mu
\medmuskip=0mu
\thickmuskip=0mu
\spaceskip=-0pt
\mathcal{L}_\text{NAL}^{\tau}(\bm{x}) = \sum_{i \in \mathcal{P}} \langle \text{sg}[\bm{F_i}^{\tau, \bm{x}} - \langle \bm{F_i}^{\tau, \bm{x}}, \hat{\bm{x}}_i \rangle \mathbf{1}], \bm{x}_i \rangle,
\label{eq:NAL}
\end{equation}
where $\boldsymbol{F}_{i}^{\tau, \boldsymbol{x}} = -\nabla_{\bm{x}_{i}} u_{i}^{\tau}(\boldsymbol{x})$, $\hat{\bm{x}} \in \bm{\mathcal{X}}$ can be any strategy profile (with $\hat{\bm{x}}_i \neq \mathbf{0}$ for all $i \in \mathcal{P}$), and $\text{sg}[\cdot]$ is the stop-gradient operator that implies the term in this operator is not involved in gradient backpropagation. So the first-order gradient of NAL is:
\begin{equation}
\setlength\abovedisplayskip{2pt}
\setlength\belowdisplayskip{2pt}
\thinmuskip=0mu
\medmuskip=0mu
\thickmuskip=0mu
\spaceskip=-0pt
\begin{aligned}
\nabla_{\boldsymbol{x}_{i}} \mathcal{L}_\text{NAL}^{\tau}(\boldsymbol{x}) &= \text{sg}\big[\boldsymbol{F}_{i}^{\tau, \boldsymbol{x}} - \langle \boldsymbol{F}_{i}^{\tau, \boldsymbol{x}}, \hat{\boldsymbol{x}}_{i} \rangle \mathbf{1}\big] \\
&= -\nabla_{\bm{x}_{i}} u_{i}^{\tau}(\boldsymbol{x}) + \langle \nabla_{\bm{x}_{i}} u_{i}^{\tau}(\boldsymbol{x}), \hat{\boldsymbol{x}}_{i} \rangle \mathbf{1}.
\end{aligned}
\end{equation}


Specifically, $\nabla_{\boldsymbol{x}_{i}} \mathcal{L}_\text{NAL}^{\tau}(\boldsymbol{x}) = \bm{0}$ if and only if the gradients for all actions are identical, which is equivalent to the global minimum of $\mathcal{L}_{G}^{\tau}(\boldsymbol{x})$. Therefore, NAL employs a single random variable to obtain an unbiased estimate of the loss function, which reduces variance and transforms the problem of finding an NE into a stochastic optimization task.

\section{Methodology}
In this section, we introduce the \textbf{T}ree-based \textbf{S}tochastic \textbf{O}ptimization (TSO) framework. First, we address the limitation of neural networks in representing non-enumerable actions by introducing the tree-based action representation. We then integrate this representation into the NAL and prove its equivalence. Additionally, we propose a sample-and-prune mechanism to mitigate the risk of converging to local optima.

\subsection{Tree-Based Action Representation}
\label{subsec:Tree-Based Action Representation}

Using a tree to model each player's action space avoids the need to enumerate all possible actions. In this approach, actions are generated through a sequence of decisions at the nodes of the tree. Since the attacker and defender use different processes to generate actions, we propose two different tree construction strategies, each specifically designed for their respective action space.

\begin{figure}[htbp]
    \centering
    \includegraphics[width=\columnwidth]{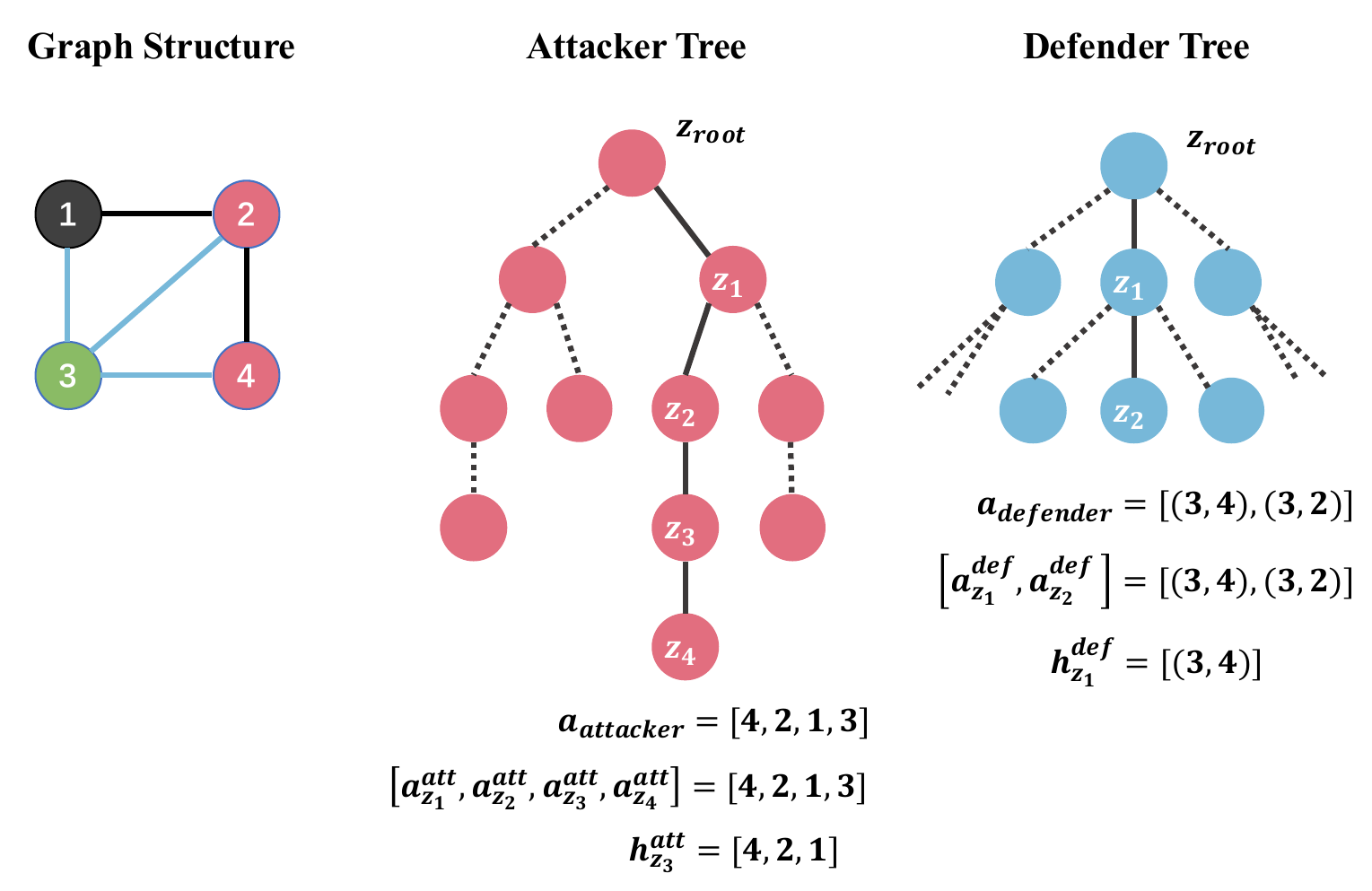}
    \caption{\small An example illustrating the construction of the attacker and defender action representation trees. Taking this graph structure as an example, red vertices are possible starting points for the attacker, and the green vertex is the attacker's target. The defense team consists of two defenders, each of whom can deploy a resource on one of the three blue edges. The attacker action [4,2,1,3] corresponds to the path from $z_{\mathrm{root}}$ to $z_4$ in the attacker tree (solid lines). Similarly, the defender action [(3,4), (3,2)] corresponds to a path from $z_{\mathrm{root}}$ to $z_2$ in the defender tree.}
    \label{fig:tree_construction}
\end{figure}

\subsubsection{Tree Construction of Attacker Action.}

The core idea of constructing a tree to represent the attacker's action space is to map each simple path in the graph $G = (V, E)$ to a unique path in the tree. We construct a tree, denoted as $\mathcal{T}_{\text{attacker}}$, which decomposes an attacker action into a series of steps, with each step corresponding to a node in $\mathcal{T}_{\text{attacker}}$. An example illustrating the construction of the attacker's action representation tree is shown in Figure~\ref{fig:tree_construction}. The tree is constructed according to the following principles:

\begin{itemize}
    \item \textbf{Tree Structure and Nodes:} Each node $z \in \mathcal{T}_{\text{attacker}}$ represents a state in the attacker's path-finding process and is defined by two components:

    \textbf{1. Action Component ($a_z^{att}$):} For any non-root node $z$, $a_z^{att} \in V$ is the vertex visited to transition from its parent to $z$. For the root node, $z_{\text{root}}$, which represents the state before any action, $a_{z_\text{root}}^{att} = \emptyset$.

    \textbf{2. History Component ($h_z^{att}$):} This component records the sequence of actions taken from $z_{\text{root}}$ to a node $z$, providing the information for subsequent decisions. Formally, for a path of nodes $(z_0, z_1, \ldots, z_i)$ from the root ($z_0 = z_{\text{root}}$) to the current node $z_i$, its history is the corresponding sequence of actions: $h_{z_i}^{att} = (a_{z_1}^{att}, a_{z_2}^{att}, \ldots, a_{z_i}^{att})$. The history of the root node is empty: $h_{z_{\text{root}}}^{att} = \emptyset$.

    \item \textbf{Children Construction:} The children of any node $z \in \mathcal{T}_{\text{attacker}}$ are determined by the valid subsequent actions from the state represented by $z$. For the root node $z_{\text{root}}$, its children correspond to all possible starting vertices in $V_{\text{start}}$. For each $v \in V_{\text{start}}$, a child node $z'$ is created with $a_{z'}^{att} = v$ and $h_{z'}^{att} = (v)$. For any non-root, non-leaf node $z$, the set of available next vertices is the set of neighbors of $a_z^{\text{att}}$ in the graph $G$, denoted $N(a_z^{att})$. To maintain the simple path constraint and ensure the path can reach a target, the set of valid next actions is filtered. Specifically, a neighbor $v'$ is considered valid if it has not been visited (i.e., $v' \notin h_z^{att}$) and there exists a simple path to a target via $v'$. The set of valid actions is denoted as $A_{\text{valid}}(z)$. For each vertex $v' \in A_{\text{valid}}(z)$, a child node $z_{\text{child}}$ is generated, with its components defined as $a_{z_{\text{child}}}^{att} = v'$ and $h_{z_{\text{child}}}^{att} = h_z^{att} \oplus (v')$, where $\oplus$ denotes sequence concatenation.

\end{itemize}

\paragraph{Path-to-Action Mapping.} This construction ensures that any path from $z_{\text{root}}$ to a leaf node $z$ corresponds to a unique attacker action $a_{\text{attacker}}$. Therefore, the process of selecting an attacker's action reduces to applying a policy to navigate from the root to a leaf in this tree. Thus, there is a one-to-one mapping between all simple paths in $G$ (from a vertex in $V_{\text{start}}$ to a vertex in $V_{\text{target}}$) and the root-to-leaf paths in $\mathcal{T}_{\text{attacker}}$.


\paragraph{Action Probability.} Based on the sequential decision-making process defined by $\mathcal{T}_{\text{attacker}}$, we now formalize the probability of an attacker's action. An attacker's action, $a_{\text{attacker}}$, is a simple path from a starting vertex to a target, represented by $a_{\text{attacker}} = (v_1, v_2, \ldots, v_L)$. This sequence corresponds to a unique root-to-leaf path $(z_0, z_1, \ldots, z_L)$ in $\mathcal{T}_{\text{attacker}}$, where $z_0=z_{\text{root}}$ and for each step $i \in \{1, \ldots, L\}$, node $z_i$ is chosen from the children of $z_{i-1}$ with its action component being $a_{z_i}^{att} = v_i$. The selection of each vertex in the sequence is controlled by a policy, $\pi_{\theta}$, parameterized by $\theta$. At each node $z_{i-1}$ on the path, the policy computes a probability distribution over all valid next actions. The probability of selecting the specific next vertex $v_i$ (and thus transitioning to node $z_i$) is conditioned on the $h_{z_{i-1}}^{att}$. This conditional probability is denoted as $\pi_{\theta}(a_{z_i}^{att} \mid h_{z_{i-1}}^{att})$. Applying the chain rule of probability, the likelihood of the $a_{\text{attacker}}$ is the product of the conditional probabilities for each decision along the corresponding path in $\mathcal{T}_{\text{attacker}}$: $\pi_{\theta}(a_{\text{attacker}}) = \prod_{i=1}^{L} \pi_{\theta}(a_{z_i}^{att} \mid h_{z_{i-1}}^{att})$

\paragraph{Tree Construction of Defender Action.} The construction of the action representation tree for the defense team is similar to that of the attacker. Each node $z$ in the $\mathcal{T}_{\text{defender}}$ consists of two components: the Action Component ($a_z^{def}$) and the History Component ($h_z^{def}$). The primary difference is that, in $\mathcal{T}_{\text{defender}}$, $a_z^{def}$ represents an edge in the graph, while $h_z^{def}$ denotes a sequence of edges. Further details regarding the construction of $\mathcal{T}_{\text{defender}}$ are provided in Appendix \ref{appendix:representation_details}.

\paragraph{Advantages.} A key advantage of our tree-based action representation is that neither $\mathcal{T}_{\text{attacker}}$ nor $\mathcal{T}_{\text{defender}}$ needs to be constructed in advance. Instead, we employ neural networks to represent the tree structure. Child nodes are generated dynamically at each node $z$, where the neural network takes the history component $h_z$ as input and makes decisions based on the historical information. Sampling only considers current-level candidates, yielding $O(d|E|)$ time per sampled action (with $d$ the maximum out-degree, $|E|$ the simple-path length bound).The output represents the probabilities of selecting each child node.  Additionally, we employ action masking to handle the variable number of child nodes. Further implementation details are in Appendix \ref{appendix:representation_details}. As a result, our method eliminates the need to enumerate the entire action space. 

\subsection{Equivalence to NAL}
\label{subsec:Equivalence to NAL}

This section formally shows how our tree-based action representation preserves the NAL property. Specifically, we demonstrate the equivalence between the tree-based and the original NAL. This ensures that optimizing the tree-based NAL shares the same objective as the original and thus satisfies the same NE conditions.

\subsubsection{Notation.}

To facilitate the proof, we first establish our notation. For each player $i \in \mathcal{P}$, the set of available actions is $\mathcal{A}_i$. We use an index $k \in \{0, 1, \dots, |\mathcal{A}_i|-1\}$ to refer to a specific action in this set. A mixed strategy for player $i$ is a probability distribution over $\mathcal{A}_i$, represented by the vector $\bm{x}_i = [\sigma_0, \sigma_1, \dots, \sigma_{|\mathcal{A}_i|-1}]$, where $\sigma_k$ is the probability of selecting the action indexed by $k$. We also consider another mixed strategy $\hat{\bm{x}}_i$ as used in the NAL.

In our method, the action probabilities $\sigma_k$ are not atomic variables but are constructed via a tree-based action representation. Each player $i$ has its own tree $\mathcal{T}_i$. Let $E_\text{tree}^i$ be the set of all edges in this tree. For each edge $j \in E_\text{tree}^i$, we define $\dot{\sigma}_j$ as the probability of traversing that edge from its parent node to the child. Each player $i$'s action $k$ corresponds to a unique root-to-leaf path in $\mathcal{T}_i$. We denote the set of edges constituting this path as $S_k \subseteq E_\text{tree}^i$. The probability of player $i$ selecting the action $k$ is the product of the traversal probabilities of all edges along this path: $\sigma_k = \prod_{j \in S_k} \dot{\sigma}_j$


\subsubsection{Tree-Based NAL.}
Building upon the tree-based action representation and Eq.(~\ref{eq:NAL}), we now derive the tree-based NAL as follows:
\begin{equation}
\setlength\abovedisplayskip{2pt}
\setlength\belowdisplayskip{2pt}
\thinmuskip=0mu
\medmuskip=0mu
\thickmuskip=0mu
\spaceskip=-0pt
\begin{aligned}
\mathcal{L}_{\text{TSO}}^\tau(\bm{x}) &= \sum_{i \in \mathcal{P}}\sum_{k=0}^{|\mathcal{A}_i|-1} \text{sg}\left[ \bm{F}_i^{\tau, \bm{x}}(k) - \langle \bm{F}_i^{\tau, \bm{x}}, \hat{\bm{x}}_i \rangle \right] \sigma_k \\
&= \sum_{i \in \mathcal{P}}\sum_{k=0}^{|\mathcal{A}_i|-1} \text{sg}\left[ \bm{F}_i^{\tau, \bm{x}}(k) - \langle \bm{F}_i^{\tau, \bm{x}}, \hat{\bm{x}}_i \rangle \right] \prod_{j \in S_k} \dot{\sigma}_j ,
\end{aligned}
\end{equation}
where $\bm{F}_i^{\tau, \bm{x}}(k)$ denotes the value of $\bm{F}_i^{\tau, \bm{x}}$ corresponding to player $i$'s $k$-th action; other terms are as in Eq.(\ref{eq:NAL}). Derivation details are provided in Appendix~\ref{appendix:derivation_TSO}.

\subsubsection{First-Order Gradient of $\mathcal{L}_{\text{TSO}}^\tau(\bm{x})$.}

The optimization of $\mathcal{L}_{\text{TSO}}^\tau$ proceeds by updating the parameters of the policy, which are the edge probabilities $\{\dot{\sigma}_j\}_{j \in E_\text{tree}^i}$. To derive the gradients for these parameters, we focus the analysis on player $i$ and consider the partial derivative of the loss with respect to one of its action probabilities $\sigma_k$, where $k \in \{0, 1, \dots, |\mathcal{A}_i|-1\}$. Due to the stop-gradient operator sg[·], which treats its argument as a constant during differentiation, the derivative simplifies to:
\begin{equation}
\setlength\abovedisplayskip{2pt}
\setlength\belowdisplayskip{2pt}
\thinmuskip=0mu
\medmuskip=0mu
\thickmuskip=0mu
\spaceskip=-0pt
\frac{\partial \mathcal{L}_\text{TSO}^\tau(\bm{x})}{\partial \sigma_k} = \text{sg}\left[ \bm{F}_i^{\tau, \bm{x}}(k) - \langle \bm{F}_i^{\tau, \bm{x}}, \hat{\bm{x}}_i \rangle \right].  \end{equation}

This term is precisely the first-order gradient for action $k$ in the NAL. For notational clarity, we define:
\begin{equation}
\setlength\abovedisplayskip{2pt}
\setlength\belowdisplayskip{2pt}
\thinmuskip=0mu
\medmuskip=0mu
\thickmuskip=0mu
\spaceskip=-0pt
    g_{i,k} := \text{sg}\left[ \bm{F}_i^{\tau, \bm{x}}(k) - \langle \bm{F}_i^{\tau, \bm{x}}, \hat{\bm{x}}_i \rangle \right].
\end{equation}
Next, we use the multivariate chain rule to propagate this gradient back to an arbitrary edge probability $\dot{\sigma}_j$:
\begin{equation}
\setlength\abovedisplayskip{2pt}
\setlength\belowdisplayskip{2pt}
\thinmuskip=0mu
\medmuskip=0mu
\thickmuskip=0mu
\spaceskip=-0pt  
\frac{\partial \mathcal{L}_\text{TSO}^\tau(\bm{x})}{\partial \dot{\sigma}_j} = \sum_{k=0}^{|\mathcal{A}_i|-1} \frac{\partial \mathcal{L}_\text{TSO}^\tau(\bm{x})}{\partial \sigma_k} \frac{\partial \sigma_k}{\partial \dot{\sigma}_j} = \sum_{k=0}^{|\mathcal{A}_i|-1} g_{i,k} \frac{\partial \sigma_k}{\partial \dot{\sigma}_j} .
\end{equation}

The partial derivative $\frac{\partial \sigma_k}{\partial \dot{\sigma}_j}$ is non-zero only if edge $j$ is on the path $S_k$. If $j \in S_k$, this derivative is the product of all other edge probabilities on the path:
\begin{equation}
\setlength\abovedisplayskip{2pt}
\setlength\belowdisplayskip{2pt}
\thinmuskip=0mu
\medmuskip=0mu
\thickmuskip=0mu
\spaceskip=-0pt  
    \frac{\partial \sigma_k}{\partial \dot{\sigma}_j} = 
\begin{cases}
\prod_{l \in S_k, l \neq j} \dot{\sigma}_l & \text{if } j \in S_k \\
0 & \text{otherwise}.
\end{cases}
\end{equation}

Let $\sigma_k^{\backslash j}$ denote $\prod_{l \in S_k, l \neq j} \dot{\sigma}_l$. The summation for the gradient thus reduces to only those actions whose paths include edge $j$. Let $\mathcal{A}_i(j) := \{k \in \{0, \dots, |\mathcal{A}_i|-1\} \mid j \in S_k\}$ be the set of indices for such actions. The final gradient expression is:
\begin{equation}
\setlength\abovedisplayskip{2pt}
\setlength\belowdisplayskip{2pt}
\thinmuskip=0mu
\medmuskip=0mu
\thickmuskip=0mu
\spaceskip=-0pt
    \frac{\partial \mathcal{L}_\text{TSO}^\tau(\bm{x})}{\partial \dot{\sigma}_j} = \sum_{k \in \mathcal{A}_i(j)} g_{i,k} \cdot \sigma_k^{\backslash j}.
\end{equation}

\begin{proposition}
For any edge $j$ in $\mathcal{T}_{i}$, the first-order gradient of the tree-based NAL equals zero if and only if the first-order gradient of NAL is $\bm{0}$, i.e.,
\label{proposition:equivalence}
\begin{equation}
\setlength\abovedisplayskip{2pt}
\setlength\belowdisplayskip{2pt}
\frac{\partial \mathcal{L}_\text{TSO}^\tau(\bm{x})}{\partial \dot{\sigma}_j} = 0 \quad \forall \, j \, \in E_{tree}^i \; \iff \; \bm{g}_i = \bm{0},
\end{equation}
where $\bm{g}_i = [g_{i,0},\, g_{i,1},\, \ldots,\, g_{i, |\mathcal{A}_i|-1}]^\top$.
The detailed proof of this proposition is provided in Appendix \ref{appendix:equivalance_proof}.
\end{proposition}

According to Proposition \ref{proposition:equivalence}, we demonstrated equivalence between $\mathcal{L}_{\text{TSO}}^\tau(\bm{x})$ and $\mathcal{L}_\text{NAL}^{\tau}(\bm{x})$ when their first-order derivatives are zero. This implies that optimization of $\mathcal{L}_{\text{TSO}}^\tau(\bm{x})$ is the same as optimization of $\mathcal{L}_\text{NAL}^{\tau}(\bm{x})$. Consequently, if the optimization of $\mathcal{L}_\text{NAL}^{\tau}(\bm{x})$ converges to an NE, it follows that the optimization of $\mathcal{L}_{\text{TSO}}^\tau(\bm{x})$ will converge to an NE.

\input{3_4_details_of_TSO}

\section{Experiments}
\label{sec:experiments}
 
To evaluate the performance and scalability of our TSO, we conduct a series of experiments on UNSGs with varying scales and complexities. We compare TSO against two baselines: PSRO and NAL.\footnote{Code repository: \url{https://github.com/sxzhuang/TSO_UNSG}.} Methods such as NSGZero/NFSP \cite{xue2022nsgzero, heinrich2016deep} (extensive-form) and EPSRO \cite{zhou2022efficient} (requires explicit action enumeration) are incompatible with our normal-form, non-enumerable setting. Our experiments are performed on a workstation with Intel i9-14900K CPU and NVIDIA RTX 4090 GPU with 24GB memory. 

 
\subsection{Experimental Setup}
 
 
\subsubsection{Game Environments.}
We design three sets of game environments based on graph size and action space: small-scale, medium-scale, and large-scale. The detailed configurations are summarized in Table~\ref{tab:game_settings}, provided in Appendix \ref{appendix:exp}.

\paragraph{Small-Scale Game (S-1).}
This experiment is conducted on a 16 nodes, 40 edges graph. It serves as a foundational test case where the action spaces are small enough, allowing us to verify the convergence properties of our algorithm in a controlled setting. The attacker has 92 possible paths, while the defense team, consisting of two defenders with 11 candidate locations each, has $11^2 = 121$ possible joint actions.
 
\paragraph{Medium-Scale Games (M-1 to M-4).}
There are four experiments which are based on a graph with 64 nodes and 300 edges and are designed to systematically evaluate the scalability of the algorithms.
\begin{itemize}
    \item Scenarios \textbf{M-1, M-2, and M-3} progressively increase the maximum path length for the attacker from 8 to 10. This expands the attacker's action space from \num{1955} to \num{20029}, while the defender's action space remains fixed at 50. This setup specifically tests the algorithms' ability to handle an increase in the attacker's action space.
    \item Scenario \textbf{M-4} builds upon the same graph structure with two defender resources. The maximum path length for the attacker is 7, resulting in 513 possible actions. The defender's action space increases exponentially from 150 to $150^2 = 22,\!500$, creating a game that challenges the defenders' tree-based action representation methods.
\end{itemize}
 
\paragraph{Large-Scale Game (L-1).}
The large-scale experiment is set on a 10000 nodes graph with 31660 edges. Given a maximum path length of 100 for the attacker in this environment, the number of possible attacking paths is so enormous that the action space cannot be feasibly enumerated. This scenario is designed as a test to evaluate the practical applicability of TSO in situations where action enumeration is infeasible.

\paragraph{Game with Asymmetric Payoffs.}
In the experimental setup described above, the attacker receives the same reward regardless of which target is reached. To test our algorithm's adaptability to more complex reward structures, we modify the S-1 setup. We configure the game with 4 distinct targets, each offering a different reward to the attacker: [1, 2, 3, 4]. The attacker's maximum path length is 9, resulting in an action space of 120 strategies. The defense team's setup remains the same as in S-1.

\begin{figure}[htbp]
    \centering
    \includegraphics[width=0.7\columnwidth]{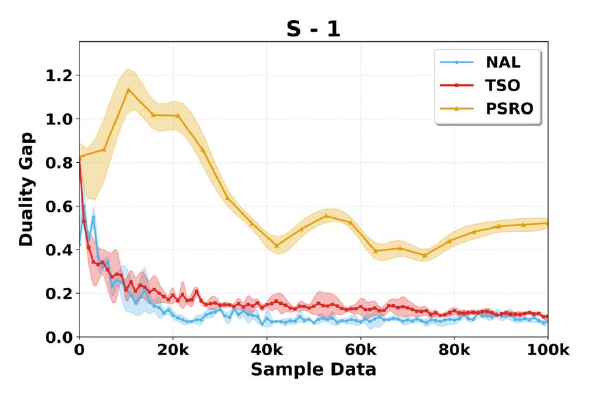}
    \caption{\small Small-Scale Game Experiment Results}
    \label{fig:small_large_exp}
\end{figure} 
 
\paragraph{Game with Decentralized Defenders.}
We investigate the performance of our algorithm under a decentralized decision-making model for the defense team. The environment is based on the S-1 experiment; however, we remove the assumption of central coordination among defenders. Each defender independently determines its placement decision without knowledge of the actions taken by other defenders, and aims to maximize the overall team utility.

\subsection{Baselines and Evaluation}
 
For games of small and medium scale where the action spaces can be enumerated, we use the \textbf{duality gap}, as introduced in the preliminaries, as our primary evaluation metric. A lower duality gap indicates a closer approximation to an NE.

\begin{figure}[htbp]
    \centering
    \includegraphics[width=\columnwidth]{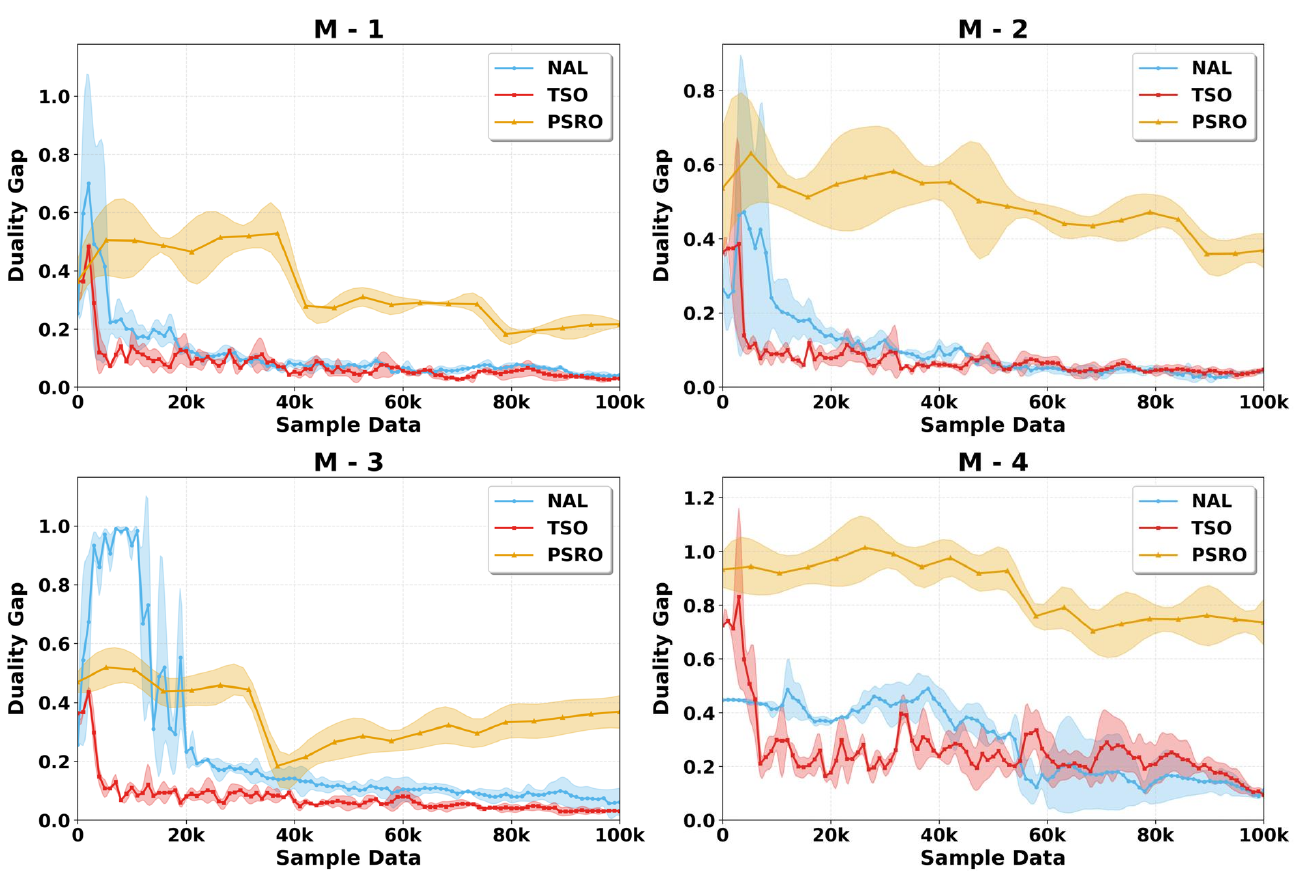}
    \caption{\small Medium-Scale Games Experiment Results}
    \label{fig:medium_exp}
\end{figure} 

We compare TSO with two baselines: NAL and PSRO. Since NAL and TSO use different network architectures, their hyperparameters were tuned separately to ensure optimal performance. Hyperparameter details are in Appendix \ref{appendix:hyper}. For each experiment, we report the mean and standard deviation over multiple runs with different random seeds.

A direct performance comparison between PSRO and NAL and TSO is non-trivial due to their different training paradigms. An ``iteration" in PSRO, which involves computing a best response and solving a meta-game, is not equivalent to a ``training step" in NAL and TSO. To establish a fair basis for comparison, we align their performance curves using the number of samples collected during training. Specifically, we constrain the total number of samples for PSRO to be identical to that of NAL and TSO. After each PSRO iteration, we evaluate the current policies by calculating their duality gap. This performance metric is then plotted against the cumulative number of samples consumed up to that iteration, enabling a direct alignment with the results of NAL and TSO. For example, as shown in Figure~\ref{fig:medium_exp}, the x-axis of the plot represents the number of sample data.
 
\subsection{Experimental Results Analysis}

\paragraph{Performance on Small- and Medium-Scale Games.}
Figure~\ref{fig:small_large_exp} illustrates that in the small-scale \textbf{S-1} scenario, TSO achieves convergence performance comparable to NAL, with only a marginal gap. This result is expected, as NAL is known to perform well in environments with smaller action spaces. The slight performance difference is acceptable, likely stemming from the increased optimization complexity introduced by the multiplicative structure in TSO.
 
As illustrated in Figure~\ref{fig:medium_exp}, across all medium-scale experiments (\textbf{M-1} to \textbf{M-4}), TSO and NAL exhibit similar duality gaps at convergence. In the case of \textbf{M-3}, TSO method demonstrates a slight advantage over NAL. In all cases, TSO converges faster than NAL. This suggests that as the dimensionality of the action space increases, NAL efficiency becomes limited. Our tree-based action representation enables more efficient optimization. The PSRO baseline's performance is substantially lower than both TSO and NAL across all S and M-series experiments. This underscores a critical limitation of PSRO in UNSGs: the bias introduced by approximating the best response accumulate over iterations, severely hindering convergence to the NE. We also evaluate the impact of the number of defenders on TSO's performance; details are in Appendix~\ref{appendix:defender_number}. Further details regarding the large-scale experiment are presented in Appendix~\ref{appendix:large_discussion}.

\begin{figure}[htbp]
    \centering
    \includegraphics[width=\columnwidth]{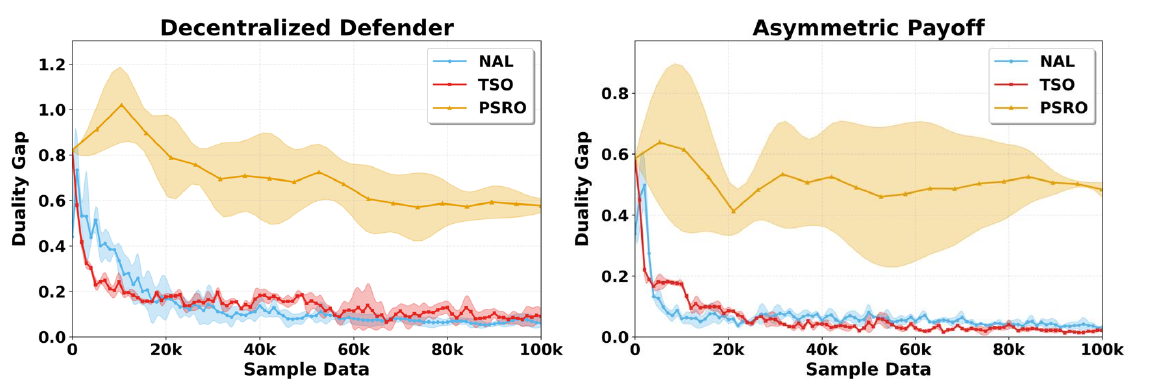}
    \caption{\small Diverse Games Experiment Results}
    \label{fig:diverse_exp}
\end{figure} 

\paragraph{Performance in Asymmetric and Decentralized Games.}
We further evaluate TSO in a game with asymmetric payoffs and another with decentralized defenders, as shown in Figure \ref{fig:diverse_exp}. In these experiments, TSO's performance remains highly competitive with NAL, and both methods maintain a significant advantage over PSRO. This demonstrates TSO framework's ability to adapt and sustain its strong performance when extending to a wider variety of game structures.

\paragraph{Hyperparameters Sensitivity Analysis and Training Time Comparison.}
We conduct hyperparameters analysis of TSO with respect to key hyperparameters. More details can be found in Appendix \ref{sec:hyperparameter_analysis}. We also compare the wall-clock time of TSO and PSRO, as shown in Appendix \ref{sec:time_comparison}. 

\subsection{Ablation Studies}
 
\paragraph{Role of Tree-Based Action Representation in TSO.}
The most significant function of our tree-based action representation is to make stochastic gradient optimization viable in UNSGs with action spaces too large to be feasibly enumerated. This results demonstrated in the \textbf{L-1} experiment, in Figure~\ref{fig:large_exp}, where TSO achieves better performance than PSRO in a setting where NAL fails to optimize. Furthermore, the performance advantage of TSO over NAL in Figure~\ref{fig:medium_exp} also highlights the benefits of our tree-based action representation.

\begin{figure}[htbp]
    \centering
    \includegraphics[width=0.6\columnwidth]{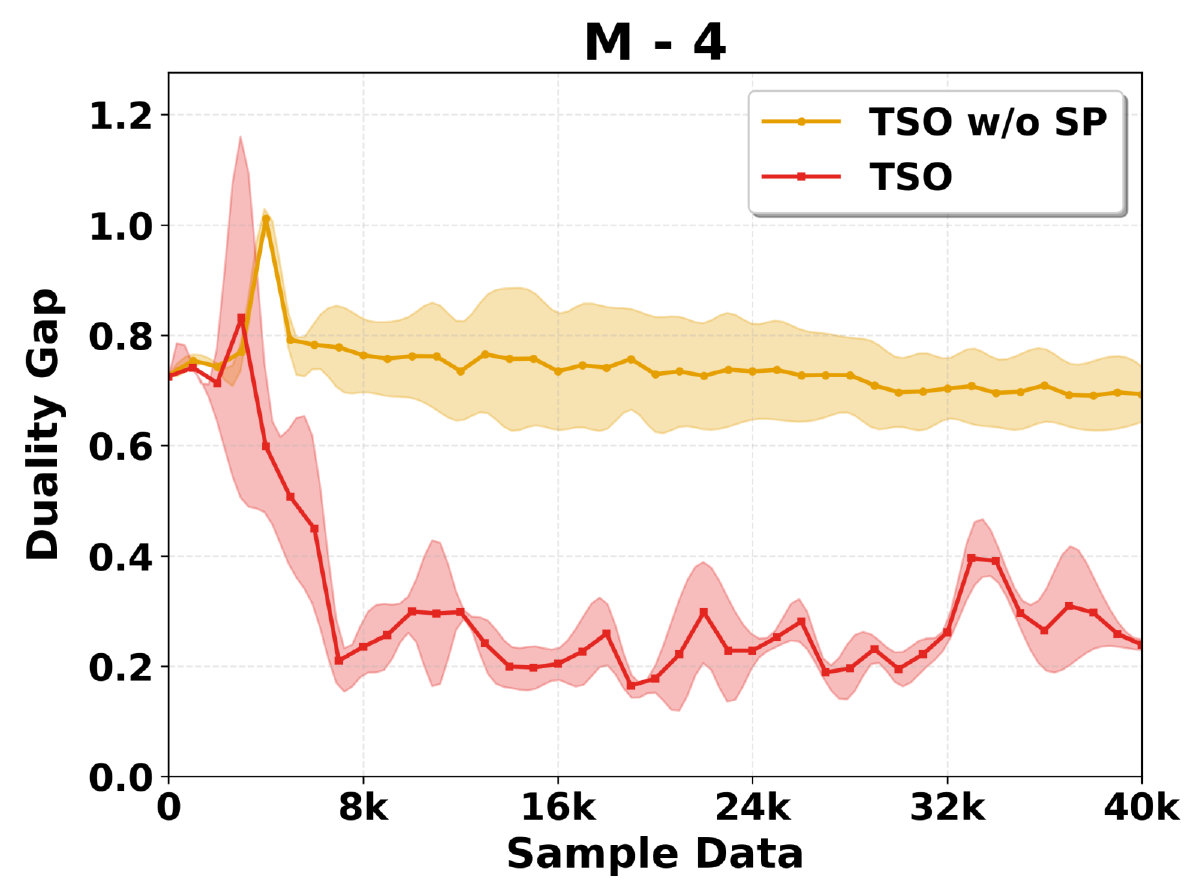}
    \caption{\small Ablation Experiment Results}
    \label{fig:ablation_exp}
\end{figure} 

\paragraph{Impact of the Sample-and-Prune Mechanism.}
To validate the Sample-and-Prune Mechanism, we performed an ablation on the \textbf{M-4} setting, comparing TSO to a variant without the mechanism (\emph{TSO w/o SP}). As shown in Figure~\ref{fig:ablation_exp}, full TSO converges to a better NE, while the ablated version gets stuck in a local optimum, confirming the Sample-and-Prune Mechanism is essential for improving TSO’s convergence.
 


\section{Conclusion}

In this paper, we proposed TSO, a framework that effectively solves large-scale UNSGs. TSO bridges the gap between stochastic optimization for NE-finding and the practical demands of games with massive and combinatorial action spaces. TSO overcomes the challenge of massive action spaces using the tree-based representation for efficient action sampling and a sample-and-prune mechanism to improve solution quality. Extensive experiments confirm that TSO outperforms existing baselines. In summary, TSO provides a scalable and effective solution for a critical class of security games and offers a promising methodology for other large-scale game-theoretic problems.
 
\section*{Acknowledgments}
This research is supported by the InnoHK Funding.

\bibliography{aaai2026}

\appendix

\onecolumn

\section{Implementation Details of Tree-Based Action Representation}
\label{appendix:representation_details}

\subsection{Defender Tree}

We decompose a joint action into an ordered sequence of individual decisions. This is achieved by constructing a tree, $\mathcal{T}_{\text{defender}}$, where a joint action is equivalent to traversing a path from the root to a leaf. The principles for constructing the tree are as follows:
\begin{itemize}
\item \textbf{Tree Structure and Nodes:} The decision tree, $\mathcal{T}_{\text{defender}}$, is structured with a depth of $N$, where each level from $1$ to $N$ corresponds to the decision made by a specific defender. A node $z \in \mathcal{T}_{\text{defender}}$ is characterized by two components:

\textbf{1. Action Component ($a_z^{def}$):} This represents the specific action that leads to node $z$ from its parent. For a node $z$ at depth $i$ ($1 \le i \le N$), its action component $a_z^{def}$ is an element of $\mathcal{E}_i$, corresponding to the choice of defender $i$. The root node $z_\text{root} = \emptyset$ has a null action component.

\textbf{2. History Component ($h_z^{def}$):} This provides a cumulative record of the decisions made along the path from the root. The history $h_z^{def}=(a_{z_1}^{def}, \ldots, a_{z_i}^{def})$ of a node $z$ at depth $i$ is the sequence of actions taken by the first $i$ defenders. The history of the root node is an empty sequence, $h_{z_{\text{root}}}^{def} = \emptyset$.
\item \textbf{Children Construction:} The tree is expanded level by level. Any non-leaf node $z$ at depth $i$ (where $0 \le i < N$) is connected to a set of children representing all possible actions for the $i+1$ defender. For each available action $e \in \mathcal{E}_{i+1}$, a child node $z_{\text{child}}^{def}$ is generated. Its action component is set to $a_{z_{\text{child}}}^{def} = e$, and its history is updated by concatenating the new action: $h_{z_{\text{child}}}^{def} = h_z^{def} \oplus (e)$.
\end{itemize}

\paragraph{Path-to-Action Mapping} Any complete path from the root to a leaf $(z_0, z_1, \ldots, z_N)$ with $z_0 = z_{\text{root}}$, uniquely specifies the joint action $a_{\text{defender}} = (a_{z_1}^{def}, \ldots, a_{z_N}^{def})$. Thus, the intractable joint action space can be mapped one-to-one onto a path from the root to a leaf in $\mathcal{T}_{\text{defender}}$.

\paragraph{Action Probability}
The tree-based action representation factorizes the joint action policy into a product of conditional probabilities. For each decision in the sequence, the policy network $\pi_\phi$, parameterized by $\phi$, determines the action probabilities. Specifically, at any non-leaf node $z \in \mathcal{T}_{\text{defender}}$ at depth $i-1$ (for $1 \le i \le N$), the policy is conditioned on the history of previously selected actions $h_z^{def}$. It then outputs a probability distribution over the action space $\mathcal{E}_i$ for the current defender. Following the chain rule of probability, the likelihood of a complete joint action $a_\text{defender}$ is the product of these sequential conditional probabilities:  $\pi_\phi(a_\text{defender}) = \prod_{i=1}^{N} \pi_\phi(a_{z_i}^{def} \mid h_{z_{i-1}}^{def})$


\subsection{Action Mask}

\paragraph{Attacker Action Mask} A challenge in our sequential decision-making framework is the policy network $\pi_\theta$ has to have a fixed-size output. To resolve this, we employ an action mask. The policy network is designed with an output dimension corresponding to the maximum node degree, $\Delta_{\max}$, in the graph. However, at any given node $z_{i-1}$ in a path, the set of valid successor nodes (children) is often smaller than $\Delta_{\max}$. The action mask is a binary vector that constrains the policy to this feasible set. Specifically, for a decision at node $z_{i-1}$, the mask identifies which dimensions of the network's output correspond to valid children. The logits for all invalid actions are then masked (e.g., set to $-\infty$) before the softmax normalization. This procedure ensures that the final probability distribution is exclusively over the set of valid actions, guaranteeing that any sampled action is valid.

This masking mechanism is what makes our tree-based sampling approach computationally tractable. By keeping the policy network's output dimension small and fixed at $\Delta_{\max}$—a value typically orders of magnitude smaller than the total action space size—we transform an intractably large, high-dimensional discrete action space into a sequence of low-dimensional decisions. This enables the assignment of explicit probabilities to each potential attacker strategy and supports efficient policy optimization, even when the total number of strategies is too large to enumerate.

\paragraph{Defender Action Masking} For the policy network of the defense team, we also need to address the challenge that the output layer size is fixed. The policy network, $\pi_\phi$, is designed to output a probability distribution over a unified action space, $\mathcal{E}_\text{defender} = \bigcup_{i=1}^{N} \mathcal{E}_i$. However, at each step $i$, the acting defender is restricted to its individual action set $\mathcal{E}_i \subseteq \mathcal{E}_\text{defender}$. To enforce this constraint, we employ an action masking mechanism.

At each decision node $z$ at depth $i-1$ in the tree $\mathcal{T}_{\text{defender}}$, we generate a binary mask vector for defender $i$. This mask has a dimension of $|\mathcal{E}_\text{defender}|$, aligning with the network's output. An entry in the mask is set to 1 if the corresponding action belongs to the set of valid choices for defender $i$, and 0 otherwise. This set of valid choices is primarily $\mathcal{E}_i$, but can be further restricted based on the history $h_z^{def}$ to prune actions that are contextually infeasible. The mask is applied to the network's raw output (logits) before the softmax operation by setting the logits of invalid actions (where the mask is 0) to $-\infty$. This procedure ensures that the policy only assigns non-zero probabilities to valid actions at each step, guaranteeing that any sampled joint action $a_\text{defender}$ is an element of the valid joint action space $\mathcal{A}_\text{defender}$.

\section{Derivation of Tree-Based NAL}
\label{appendix:derivation_TSO}

In this section, we present a detailed derivation of Tree-Based NAL.

\textbf{Step 1:} Start with the NAL Loss Function

We begin with the definition of $\mathcal{L}_\text{NAL}^{\tau}(\bm{x})$:
$$
\mathcal{L}_\text{NAL}^{\tau}(\bm{x}) = \sum_{i \in \mathcal{P}} \langle \text{sg}[\bm{F_i}^{\tau, \bm{x}} - \langle \bm{F_i}^{\tau, \bm{x}}, \hat{\bm{x}}_i \rangle \mathbf{1}], \bm{x}_i \rangle
$$
The notation $\langle \bm{A}, \bm{B} \rangle$ represents the inner product of two vectors $\bm{A}$ and $\bm{B}$.

\textbf{Step 2:} Expand the Inner Product

The core of the derivation is to expand the inner product into its element-wise summation form. Let's focus on the term inside the summation over players $i \in \mathcal{P}$. Define two vectors:
$\bm{A} = \text{sg}[\bm{F_i}^{\tau, \bm{x}} - \langle \bm{F_i}^{\tau, \bm{x}}, \hat{\bm{x}}_i \rangle \mathbf{1}]$ and $\bm{B} = \bm{x}_i$.
These two vectors have $|\mathcal{A}_i|$ elements. Therefore, we will sum from $k=0$ to $|\mathcal{A}_i|-1$. The inner product becomes:
$$
\sum_{k=0}^{|\mathcal{A}_i|-1} A(k) \cdot B(k)
$$

\textbf{Step 3:} The Components of the Vectors

For vector $\bm{A} = \text{sg}[\bm{F_i}^{\tau, \bm{x}} - \langle \bm{F_i}^{\tau, \bm{x}}, \hat{\bm{x}}_i \rangle \mathbf{1}]$: The term $\langle \bm{F_i}^{\tau, \bm{x}}, \hat{\bm{x}}_i \rangle$ is a scalar value (a single number). The vector $\mathbf{1}$ is a vector of ones. Therefore, the $k$-th element of the vector $\bm{F_i}^{\tau, \bm{x}} - \langle \bm{F_i}^{\tau, \bm{x}}, \hat{\bm{x}}_i \rangle \mathbf{1}$ is the $k$-th element of $\bm{F_i}^{\tau, \bm{x}}$ minus this scalar. Then we get:
$$A(k) = \text{sg}\left[ \bm{F}_i^{\tau, \bm{x}}(k) - \langle \bm{F_i}^{\tau, \bm{x}}, \hat{\bm{x}}_i \rangle \right]$$.

For vector $\bm{B} = \bm{x}_i$, the $k$-th element of $\bm{x}_i$ is $\sigma_k$. So $B(k) = (\bm{x}_i)_k = \sigma_k$. 

Therefore, we have:
$$
\langle \text{sg}[\bm{F_i}^{\tau, \bm{x}} - \langle \bm{F_i}^{\tau, \bm{x}}, \hat{\bm{x}}_i \rangle \mathbf{1}], \bm{x}_i \rangle = \sum_{k=0}^{|\mathcal{A}_i|-1} \text{sg}\left[ \bm{F}_i^{\tau, \bm{x}}(k) - \langle \bm{F_i}^{\tau, \bm{x}}, \hat{\bm{x}}_i \rangle \right] \cdot \sigma_k
$$

\textbf{Step 4:} Integrate Tree-Based Action Representation

Given the Tree-Based Action Representation, the definition for $\sigma_k$ is 
$\sigma_k = \prod_{j \in S_k} \dot{\sigma}_j$.

Substituting this expression for $\sigma_k$ into the equation from Step 3 yields:
$$
\sum_{k=0}^{|\mathcal{A}_i|-1} \text{sg}\left[ \bm{F}_i^{\tau, \bm{x}}(k) - \langle \bm{F_i}^{\tau, \bm{x}}, \hat{\bm{x}}_i \rangle \right] \cdot \sigma_k = \sum_{k=0}^{|\mathcal{A}_i|-1}\text{sg}\left[ \bm{F}_i^{\tau, \bm{x}}(k) - \langle \bm{F}_i^{\tau, \bm{x}}, \hat{\bm{x}}_i \rangle \right] \prod_{j \in S_k} \dot{\sigma}_j
$$

This completes the derivation. This results in the final expression for the tree-based NAL:

$$\begin{aligned}
\mathcal{L}_{\text{TSO}}^\tau(\bm{x})
&= \sum_{i \in \mathcal{P}}\left\langle \text{sg}\left[\bm{F}_i^{\tau, \bm{x}} - \langle \bm{F}_i^{\tau, \bm{x}}, \hat{\bm{x}}_i \rangle \mathbf{1}\right], \bm{x}_i \right\rangle \\
&= \sum_{i \in \mathcal{P}}\sum_{k=0}^{|\mathcal{A}_i|-1} \text{sg}\left[ \bm{F}_i^{\tau, \bm{x}}(k) - \langle \bm{F}_i^{\tau, \bm{x}}, \hat{\bm{x}}_i \rangle \right] \sigma_k \\
&= \sum_{i \in \mathcal{P}}\sum_{k=0}^{|\mathcal{A}_i|-1} \text{sg}\left[ \bm{F}_i^{\tau, \bm{x}}(k) - \langle \bm{F}_i^{\tau, \bm{x}}, \hat{\bm{x}}_i \rangle \right] \prod_{j \in S_k} \dot{\sigma}_j
\end{aligned}$$

\section{Proof of Proposition 1}
\label{appendix:equivalance_proof}

\begin{proposition*}
\textbf{\ref{proposition:equivalence}}
For any edge $j$ in $\mathcal{T}_{i}$, the first-order gradient of the tree-based NAL equals zero if and only if the first-order gradient of NAL is $\bm{0}$, i.e.,
\begin{equation}
\setlength\abovedisplayskip{2pt}
\setlength\belowdisplayskip{2pt}
\frac{\partial \mathcal{L}_\text{TSO}^\tau(\bm{x})}{\partial \dot{\sigma}_j} = 0 \quad \forall \, j \, \in E_{tree}^i \; \iff \; \bm{g}_i = \bm{0},
\end{equation}
where $\bm{g}_i = [g_{i,0},\, g_{i,1},\, \ldots,\, g_{i, |\mathcal{A}_i|-1}]^\top$.
\end{proposition*}

\begin{proof}
We prove the equivalence by demonstrating both directions of the implication.

($\implies$) Proof that if all edge gradients are zero, then the first-order gradient of NAL is zero vector.

Assume that $\frac{\partial \mathcal{L}_\text{TSO}^\tau(\bm{x})}{\partial \dot{\sigma}_j} = 0$ for all edges $j \in E_\text{tree}^i$. The proof for this direction leverages the unique path structure of the decision tree $\mathcal{T}_i$. Consider an arbitrary action $k$. This action corresponds to a unique root-to-leaf path $S_k$. Let $j_k^{\text{term}}$ denote the terminal edge of this path, which is the final edge connecting to the leaf node in this path. By the definition of a tree, this terminal edge is exclusive to path $S_k$; no other path $S_{k'}$ for $k' \neq k$ contains $j_k^{\text{term}}$. Consequently, the set of action indices whose paths include this edge, $\mathcal{A}_i(j_k^{\text{term}})$, contains only the single element $\{k\}$.

By our initial assumption, the gradient for this specific terminal edge is zero:
$$\frac{\partial \mathcal{L}_\text{TSO}^\tau(\bm{x})}{\partial \dot{\sigma}_{j_k^{\text{term}}}} = \sum_{k' \in \mathcal{A}_i(j_k^{\text{term}})} g_{i,k'} \cdot \sigma_{k'}^{\backslash j_k^{\text{term}}} = 0$$
Since $\mathcal{A}_i(j_k^{\text{term}}) = \{k\}$, the summation collapses to a single term:
$$g_{i,k} \cdot \sigma_k^{\backslash j_k^{\text{term}}} = 0$$

For any action $k$ with a non-zero probability ($\sigma_k > 0$), all its constituent edge probabilities $\{\dot{\sigma}_j\}_{j \in S_k}$ must be positive. It follows that $\sigma_k^{\backslash j_k^{\text{term}}}$, being a product of a subset of these positive probabilities, must also be positive. Therefore, the equation can only hold if: $g_{i,k} = 0$

This logic applies to any action $k \in \{0, \dots, |\mathcal{A}_i|-1\}$. Since every action corresponds to a unique leaf node and thus has a unique terminal edge, we can apply this argument to each action's terminal edge to show that $g_{i,k}=0$ for all $k$. Thus, if the gradients for all edge parameters are zero, the first-order gradient of NAL $[g_{i,0}, \dots, g_{i, |\mathcal{A}_i|-1}]^\top$ must be the zero vector.

($\impliedby$) Proof that if the first-order gradient of NAL is zero vector, then all edge gradients are zero.

This direction is straightforward. Assume the first-order gradient of NAL is zero vector, meaning $g_{i,k} = 0$ for all action indices $k \in \{0, \dots, |\mathcal{A}_i|-1\}$. Recall the general formula for the gradient with respect to an arbitrary edge probability $\dot{\sigma}_j$:
$$\frac{\partial \mathcal{L}_\text{TSO}^\tau(\bm{x})}{\partial \dot{\sigma}_j} = \sum_{k \in \mathcal{A}_i(j)} g_{i,k} \cdot \sigma_k^{\backslash j}$$
Since we have assumed $g_{i,k} = 0$ for all $k$, every term in the summation is zero.
$$\frac{\partial \mathcal{L}_\text{TSO}^\tau(\bm{x})}{\partial \dot{\sigma}_j} = \sum_{k \in \mathcal{A}_i(j)} (0) \cdot \sigma_k^{\backslash j} = 0$$
This holds for any edge $j \in E_\text{tree}^i$. Therefore, if the first-order gradient of NAL is zero vector, the gradients for all of our underlying edge probability parameters must also be zero.

Having proven both directions, we conclude that the stationary points of our TSO loss function correspond directly to the stationary points of the NAL objective.
\end{proof}

\section{Details about Sample-and-Prune Mechanism}
\label{appendix:sample}

\paragraph{Failure of Naive Sampling in UNSGs} Consider the following problem. We use an neural network (NN) with parameters $\theta$, where the input is $x$ and the output is a $N$-dimensional vector. After applying the Softmax function, we obtain a probability distribution $Y = \text{softmax}(f(x; \theta))$, such that $Y$ is a probability vector and $\sum_{i=0} Y[i] = 1$. 

For each input/output pair $(x_1, x_2)$, the network produces $(Y_1, Y_2)$. Action sampling is performed by drawing actions (indices) $k_1$ and $k_2$ from the categorical distributions defined by $Y_1$ and $Y_2$, respectively, i.e., $k_1 \sim \text{Categorical}(Y_1)$ and $k_2 \sim \text{Categorical}(Y_2)$. The probability of selecting $k_1$ is $q(x_1) = Y_1[k_1]$ and the probability of selecting $k_2$ is $q(x_2) = Y_2[k_2]$. The objective is to maximize $\frac{(r-\bar{r})}{p} \cdot q(x_1) q(x_2)$ or, equivalently, to minimize the loss $L(\theta) = \frac{(\bar{r}-r)}{p} \cdot q(x_1) q(x_2)$, where $p$ is the probability from an $\epsilon$-greedy sampling strategy, which is detached from the gradient computation.

Since $\frac{(\bar{r}-r)}{p}$ does not participate in the differentiation, we represent it as a constant $C$. Let $z_1$ denote the pre-Softmax output (logits) of the network, and $Y_1 = \text{softmax}(z_1)$. Then, the derivative of $Y_1[k_1]$ with respect to its corresponding logit $z_{1}[k_1]$ is given by
$$\frac{\partial Y_1[k_1]}{\partial z_{1}[k_1]} = Y_1[k_1] \cdot (1 - Y_1[k_1]).$$
This derivative is a key component of the gradient $\nabla_\theta Y_1[k_1]$ via the chain rule. We can see that when $Y_1[k_1] \to 1$, $(1 - Y_1[k_1]) \to 0$, so $\frac{\partial Y_1[k_1]}{\partial z_{1}[k_1]} \to 0$; when $Y_1[k_1] \to 0$, $\frac{\partial Y_1[k_1]}{\partial z_{1}[k_1]} \to 0$ as well. 

If, over several consecutive batches, the sampled $(k_1, k_2)$ happen to receive very small $C$ (good results), the network will experience strong positive feedback, and gradient updates will keep increasing $Y_1[k_1]$ and $Y_2[k_2]$. This positive feedback loop causes $Y_1[k_1]$ and $Y_2[k_2]$ to rapidly approach 1. Once $Y_1[k_1] \approx 1$, the output distribution $Y_1$ becomes very ``sharp'' (low-entropy), with almost all probability mass concentrated on the index $k_1$. 

At this point, a problem arises: even if a subsequently sampled $(k_1, k_2)$ receives a very large $C$, and we would expect the network to quickly reduce $Y_1[k_1]$, the gradient $\nabla_\theta Y_1[k_1]$ becomes extremely small because $\frac{\partial Y_1[k_1]}{\partial z_{1}[k_1]} \approx 0$. The large $C$ multiplied by this nearly zero gradient results in a negligible parameter update $\Delta \theta$, making it almost impossible for the network to learn from this new, potentially corrective data.

This phenomenon is particularly pronounced in UNSGs. The reason is that, both the attacker and the defender have vast action spaces. The disparity between the typically small batch size used for training and the large-scale of the action space leads to high variance in the sampled action-reward pairs. Consequently, it becomes highly probable that, "the sampled actions happen to receive good results, the network will experience strong positive feedback, and gradient updates will keep increasing." This premature positive feedback loop causes the policy to converge to a local optimum. Critically, once trapped, even a high exploration rate (e.g., a large $\epsilon$ in an $\epsilon$-greedy strategy) is often insufficient for the policy to escape this suboptimal state.

This problem is addressed with an additional constraint: first, we sample an action pair $(k_1, k_2)$; then, we resample a different pair $(k_1', k_2')$ such that $(k_1', k_2') \neq (k_1, k_2)$; and finally, we use the latter pair to compute the loss. When the probability $q(k_1) \cdot q(k_2)$ of an action pair $(k_1, k_2)$ becomes high, its chance of being selected in the first sampling increases accordingly. As a result, the probability that it is excluded from gradient updates also rises. Consequently, gradient updates are applied more frequently to other action pairs with lower probabilities. The increase in the probabilities of these other action pairs effectively ``steals'' probability mass from $(k_1, k_2)$, thus preventing the unbounded growth of $q(k_1) \cdot q(k_2)$.

\paragraph{Why the Sampling-and-Pruning Mechanism Works} The efficacy of this mechanism lies in its self-regulating nature. When the probability of a specific action $a_k$ becomes high, it is more likely to be selected in the initial sampling step. According to our mechanism, being selected in this first step prunes it from the set of candidates for the gradient update in that iteration. As a result, gradient updates are more frequently applied to other, less probable actions. The reinforcement of these alternative actions causes their probabilities to rise, which "steals" probability mass from the dominant action $a_k$. This process inherently prevents any single action's probability from growing uncontrollably towards 1, thus ensuring the network remains responsive to both positive and negative feedback across the entire action space.

\paragraph{Sample-and-Prune Mechanism Does Not Affect the NE Convergence}

\begin{proposition*}
\textbf{\ref{sampling_proposition}}
The use of the Sample-and-Prune mechanism for action sampling does not affect the equivalence between $\mathcal{L}_\text{TSO}^\tau(\bm{x})$ and $\mathcal{L}_{\text{NAL}}^{\tau}(\bm{x})$.
\end{proposition*}

\begin{proof}
The NAL algorithm relies on a sampling policy, denoted as $\hat{\bm{x}}_i$ for player $i$, to guide exploration; this is commonly an $\epsilon$-greedy policy. Player $i$'s strategy, $\bm{x}_i$, is based on a first-order gradient estimator constructed from the utility of actions drawn from $\hat{\bm{x}}_i$. 

The Sample-and-Prune mechanism alters this sampling distribution. Under this mechanism, the probability of sampling an action $a_{k'}$ for the update is conditioned on a different, pre-sampled action $a_k$. While this modification changes the sampling probabilities compared to a standard $\epsilon$-greedy approach, the convergence of NAL is fundamentally agnostic to the particulars of the sampling policy \cite{mengreducing}. As the Sample-and-Prune mechanism only changes the sampling strategy, the theoretical guarantee of convergence to an NE remains intact.
\end{proof}

\input{6_appendxix_outline}

\section{Notation}
\label{appendix:tso_detail}

\begin{table*}[htbp]
\centering
\begin{tabular}{|>{\centering\arraybackslash}p{3.5cm}|p{9cm}|}
\hline
\textbf{Notation/Definition} & \textbf{Description} \\
\hline
$G = (V, E)$ & The game graph, where $V$ are intersections and $E$ are roads. \\
\hline
$V$ & The set of vertices (intersections). \\
\hline
$E$ & The set of edges (roads). \\
\hline
$N$ & The number of defenders in the defense team. \\
\hline
$V_{\text{start}}$ & The set of starting vertices for the attacker. \\
\hline
$V_{\text{target}}$ & The set of target vertices for the attacker. \\
\hline
$a_\text{attacker}$ & An action for the attacker, which is a path from a start to a target vertex. \\
\hline
$\mathcal{A}_\text{attacker}$ & The attacker's action space (the set of all valid simple paths). \\
\hline
$m$ & An index for a defender, $m \in \{1, \dots, N\}$. \\
\hline
$e_m$ & An edge selected by defender $m$. \\
\hline
$\mathcal{E}_m$ & The individual action space for defender $m$, $\mathcal{E}_m \subseteq E$. \\
\hline
$a_\text{defender}$ & A team action for the defenders, a tuple of selected edges. \\
\hline
$\mathcal{A}_\text{defender}$ & The defense team's action space, $\times_{i=1}^{N} \mathcal{E}_i$. \\
\hline
$U$ & The target value function, $U: V_{\text{target}} \to \mathbb{R}^+$. \\
\hline
$\text{target}(a_\text{attacker})$ & The endpoint of the attacker's chosen path. \\
\hline
$u_\text{attacker}$ & The attacker's utility function. \\
\hline
$u_\text{defender}$ & The defender's utility function. \\
\hline
$\mathcal{P}$ & The set of players, $\{\text{attacker}, \text{defender}\}$. \\
\hline
$i$ & An index for a player, $i \in \mathcal{P}$. \\
\hline
$\bm{x}_i$ & A mixed strategy for player $i$. \\
\hline
$\mathcal{A}_i$ & The pure action set for player $i$. \\
\hline
$\bm{\mathcal{X}}_i$ & The space of all valid mixed strategies for player $i$. \\
\hline
$\bm{x}$ & A strategy profile for the game. \\
\hline
$\bm{\mathcal{X}}$ & The joint strategy space. \\
\hline
$\bm{\mathcal{X}}^\circ$ & The interior of the joint strategy space. \\
\hline
$u_i(\bm{x})$ & The expected utility for player $i$ given strategy profile $\bm{x}$. \\
\hline
$\bm{x}_{-i}$ & The strategies of all players except player $i$. \\
\hline
$\mathrm{dg}(\boldsymbol{x})$ & The duality gap of a strategy profile $\boldsymbol{x}$. \\
\hline
$\bm{x}^*$ & A Nash equilibrium strategy profile. \\
\hline
$u_{i}^{\tau}(\boldsymbol{x})$ & Entropy-regularized utility for player $i$. \\
\hline
$\tau$ & A positive regularization parameter. \\
\hline
$\boldsymbol{F}_{i}^{\tau, \boldsymbol{x}}$ & Negative gradient of the regularized utility, $-\nabla_{\bm{x}_{i}} u_{i}^{\tau}(\boldsymbol{x})$. \\
\hline
$\overline{\boldsymbol{F}}_{i}^{\tau, \boldsymbol{x}}$ & The mean of $\boldsymbol{F}_{i}^{\tau, \boldsymbol{x}}$ over player $i$'s actions. \\
\hline
$\hat{\bm{x}}$ & An arbitrary strategy profile used in the NAL loss. \\
\hline
\end{tabular}
\end{table*}

\begin{table*}[htbp]
\centering
\begin{tabular}{|>{\centering\arraybackslash}p{3.5cm}|p{9cm}|}
\hline
\textbf{Notation/Definition} & \textbf{Description} \\
\hline
$\mathcal{T}_{\text{attacker}}$ & The decision tree constructed to represent the attacker's action space. \\
\hline
$a_z^{att}$ & The action component of a node $z$, representing the vertex visited. \\
\hline
$a_{z_\text{root}}^{att}$ & The action component of the root node, defined as $\emptyset$. \\
\hline
$h_z^{att}$ & The history component of a node $z$, recording the sequence of actions. \\
\hline
$N(a_z^{att})$ & The set of neighbors of the vertex $a_z^{att}$ in the graph $G$. \\
\hline
$d(v', V_{\text{target}})$ & The shortest distance from a vertex $v'$ to any target in $V_{\text{target}}$. \\
\hline
$A_{\text{valid}}(z)$ & The set of valid next actions from the state represented by node $z$. \\
\hline
$z_{\text{child}}$ & A child node generated for a valid action. \\
\hline
$\pi_{\theta}(a_{\text{attacker}})$ & The probability of a specific attacker action $a_{\text{attacker}}$. \\
\hline
$\Delta_{\max}$ & The maximum node degree in the graph. \\
\hline
$\mathcal{T}_{\text{defender}}$ & The decision tree representing the defender's joint action space. \\
\hline
$N$ & The number of defenders (the depth of the tree). \\
\hline
$a_z^{def}$ & The action component of node $z$ (the action taken to reach $z$). \\
\hline
$h_z^{def}$ & The history component of node $z$, i.e., the sequence $(a_{z_1}^{def},\ldots,a_{z_i}^{def})$. \\
\hline
$\pi_\phi$ & The defenders' policy, parameterized by $\phi$. \\
\hline
$\phi$ & The parameters of the defenders' policy network. \\
\hline
$\pi_\phi(a_\text{defender})$ & The probability of a joint defender action $a_\text{defender}$. \\
\hline
$k$ & An index for a specific action in a player's action set. \\
\hline
$\sigma_k$ & The probability of player $i$ selecting action $k$, constructed from edge probabilities. \\
\hline
$\mathcal{T}_i$ & The decision tree for player $i$. \\
\hline
$E_\text{tree}^i$ & The set of all edges in the decision tree $\mathcal{T}_i$. \\
\hline
$\dot{\sigma}_j$ & The fundamental probability of traversing a specific edge $j$ in a decision tree. \\
\hline
$S_k$ & The set of edges in $\mathcal{T}_i$ that form the unique path for action $k$. \\
\hline
$\mathcal{L}_{\text{TSO}}^\tau(\bm{x})$ & The Tree-based Sampling Optimization (TSO) loss function. \\
\hline
$\bm{F}_i^{\tau, \bm{x}}(k)$ & The $k$-th component of the vector $\bm{F}_i^{\tau, \bm{x}}$. \\
\hline
$\frac{\partial \mathcal{L}_\text{TSO}^\tau(\bm{x})}{\partial \sigma_k}$ & The partial derivative of the TSO loss with respect to an action probability $\sigma_k$. \\
\hline
$g_{i,k}$ & A shorthand for the first-order gradient of the NAL loss for player $i$'s action $k$. \\
\hline
$\frac{\partial \mathcal{L}_\text{TSO}^\tau(\bm{x})}{\partial \dot{\sigma}_j}$ & The partial derivative of the TSO loss with respect to an edge probability $\dot{\sigma}_j$. \\
\hline
$\frac{\partial \sigma_k}{\partial \dot{\sigma}_j}$ & The partial derivative of an action probability $\sigma_k$ w.r.t. an edge probability $\dot{\sigma}_j$. \\
\hline
$\sigma_k^{\backslash j}$ & The product of all edge probabilities on path $S_k$ except for edge $j$. \\
\hline
$\mathcal{A}_i(j)$ & The set of action indices whose corresponding paths in $\mathcal{T}_i$ include edge $j$. \\
\hline
$j_k^{\text{term}}$ & The terminal edge of the path $S_k$, which is unique to that path. \\
\hline
\end{tabular}
\end{table*}

\section{Experiments}
\label{appendix:exp}

This section provides further details on our experiments, including the specific game settings, hyperparameter configurations, and a discussion on the results of large-scale UNSGs.

\subsection{Experiment Settings}

To provide a clearer overview and highlight the key distinctions, we summarize the specific configuration of each experiment in Table \ref{tab:game_settings}.

\begin{table*}[!ht]
\centering
\caption{Summary of Experimental Game Environment Configurations. $|A_A|$ and $|A_D|$ denote the size of the attacker's and defender's action spaces, respectively.}
\label{tab:game_settings}
\sisetup{group-separator={,}} 
\begin{tabular}{@{}llccS[table-format=6.0]cc>{\centering\arraybackslash}p{2.5cm}@{}}
\toprule
\textbf{Scale} & \textbf{Scenario ID} & \multicolumn{2}{c}{\textbf{Graph Structure}} & \multicolumn{2}{c}{\textbf{Attacker Setup}} & \multicolumn{2}{c}{\textbf{Defense Team Setup}} \\
\cmidrule(lr){3-4} \cmidrule(lr){5-6} \cmidrule(lr){7-8}
& & \#Nodes & \#Edges & {\#Exits, Max Path Len.} & \multicolumn{1}{c}{$|A_A|$} & {\#Defenders, \#Locations} & \multicolumn{1}{c}{$|A_D|$} \\
\midrule
\textbf{Small} & S-1 & 16 & 40 & {1 exit, Len. 9} & 92 & {2 def, 11 loc.} & 121 \\
\midrule
\multirow{4}{*}{\textbf{Medium}} 
& M-1 & \multirow{4}{*}{64} & \multirow{4}{*}{300} & {4 exits, Len. 8} & 1955 & {1 def, 150 loc.} & 150 \\
& M-2 & & & {4 exits, Len. 9} & 6468 & {1 def, 150 loc.} & 150 \\
& M-3 & & & {4 exits, Len. 10} & 20029 & {1 def, 150 loc.} & 150 \\
& M-4 & & & {4 exits, Len. 7} & 513 & {2 def, 150 loc.} & \num{22500} \\
\midrule
\textbf{Large} & L-1 & 10000 & \num{31660} & \multicolumn{2}{c}{cannot feasibly enumerate} & {1 def, \num{31660} loc.} & {$31660$} \\ 
\bottomrule
\end{tabular}
\end{table*}

\subsection{Hyperparameters Settings}
\label{appendix:hyper}

We perform hyperparameter tuning for the TSO, PSRO, and NAL methods using a grid search to identify the optimal parameter combinations. The resulting configurations are detailed in Table \ref{tab:TSO_hyper} for TSO, Table \ref{tab:PSRO_hyper} for PSRO, and Table \ref{tab:NAL_hyper} for NAL.

\begin{table*}[htbp]
\centering
\caption{Hyperparameter configurations for our TSO experiments.}
\label{tab:hyperparameters_3col_grouped}
\begin{tabular}{|l|c|p{8.5cm}|}
\hline
\textbf{Hyperparameter} & \textbf{Value} & \textbf{Description} \\
\hline
\texttt{seed} & 3407 & The random seed for ensuring reproducibility. \\
\hline
\texttt{total\_epoch} & 50,000 & The total number of training epochs. \\
\hline
\texttt{batch\_num} & 100 & The number of episodes collected for each training batch. \\
\hline
\texttt{lr\_pursuer} & 1e-4 & The initial learning rate for the defender's network. \\
\hline
\texttt{lr\_evader} & 1e-4 & The initial learning rate for the evader's network. \\
\hline
\texttt{weight\_pursuer\_lr} & 0.8 & The multiplicative decay factor for the defender's learning rate. \\
\hline
\texttt{weight\_evader\_lr} & 0.8 & The multiplicative decay factor for the evader's learning rate. \\
\hline
\texttt{tau} & 0.05 & The initial coefficient for the entropy regularization term. \\
\hline
\texttt{weight\_tau} & 0.7 & The multiplicative decay factor for the temperature \texttt{tau}. \\
\hline
\texttt{update\_percentage} & 0.01 & The interval for decaying learning rates and \texttt{tau}, as a percentage of total epochs. \\
\hline
\texttt{epsilon} & 0.8 & The initial value for the $\epsilon$-greedy exploration strategy. \\
\hline
\texttt{update\_epsilon} & False & A flag to enable the decay of \texttt{epsilon} during training. \\
\hline
\texttt{max\_cpu\_processes} & 4 & The number of parallel CPU processes used for data collection. \\
\hline
\end{tabular}
\label{tab:TSO_hyper}
\end{table*}

\begin{table*}[htbp]
\centering
\caption{Hyperparameter configurations for the PPO-based PSRO experiments.}
\label{tab:hyperparameters_psro_ppo}
\begin{tabular}{|l|c|p{8.5cm}|}
\hline
\textbf{Hyperparameter} & \textbf{Value} & \textbf{Description} \\
\hline
\texttt{seed} & 3407 & The random seed for ensuring reproducibility. \\
\hline
\texttt{num\_psro\_iteration} & 20 & The total number of PSRO iterations. \\
\hline
\texttt{train\_evader\_number} & 20,000 & Training iterations to compute the evader's best response. \\
\hline
\texttt{train\_pursuer\_number} & 20,000 & Training iterations to compute the defender's best response. \\
\hline
\texttt{eval\_episodes} & 1,000 & Number of episodes used for evaluating policies to build the payoff matrix. \\
\hline
\texttt{lr\_pursuer} & 1e-4 & The learning rate for the defender's optimizer. \\
\hline
\texttt{lr\_evader} & 1e-4 & The learning rate for the evader's optimizer. \\
\hline
\texttt{anneal\_lr} & True & If True, the learning rate is decayed linearly over the training process. \\
\hline
\texttt{num\_steps} & 100 & The number of steps to run in each environment per policy rollout. \\
\hline
\texttt{max\_cpu\_processes} & 1 & The number of parallel CPU processes. \\
\hline
\texttt{gamma} ($\gamma$) & 0.99 & The discount factor. \\
\hline
\texttt{gae\_lambda} ($\lambda$) & 0.95 & The lambda for General Advantage Estimation (GAE). \\
\hline
\texttt{update\_epochs} & 2 & The number of epochs to update the policy for each batch of data. \\
\hline
\texttt{num\_minibatches} & 4 & The number of mini-batches to split the rollout data into. \\
\hline
\texttt{norm\_adv} & True & If True, the advantages are normalized. \\
\hline
\texttt{clip\_coef} & 0.2 & The surrogate clipping coefficient for the PPO objective. \\
\hline
\texttt{clip\_vloss} & True & If True, a clipped loss is used for the value function. \\
\hline
\texttt{ent\_coef} & 0.01 & The coefficient for the entropy bonus in the loss function. \\
\hline
\texttt{vf\_coef} & 0.5 & The coefficient for the value function loss. \\
\hline
\texttt{max\_grad\_norm} & 0.5 & The maximum norm for gradient clipping. \\
\hline
\texttt{target\_kl} & None & The target KL divergence threshold for early stopping. \\
\hline
\end{tabular}
\label{tab:PSRO_hyper}
\end{table*}

\begin{table*}[htbp]
\centering
\caption{Hyperparameter configurations for NAL method.}
\label{tab:hyperparameters_from_call}
\begin{tabular}{|l|c|p{8.5cm}|}
\hline
\textbf{Hyperparameter} & \textbf{Value} & \textbf{Description} \\
\hline
\texttt{seed} & 3407 & The random seed for ensuring reproducibility. \\
\hline
\texttt{total\_epochs} & 50,000 & The total number of training epochs. \\
\hline
\texttt{episodes\_num} & 100 & The number of episodes collected in each iteration or batch. \\
\hline
\texttt{print\_fq} & 100 & The frequency (in epochs) at which training progress is printed. \\
\hline
\texttt{initial\_lr} & 1e-4 & The initial learning rate for the optimizer. \\
\hline
\texttt{initial\_tau} & 0.1 & The initial temperature coefficient for entropy regularization. \\
\hline
\texttt{weight\_lr} & 0.9 & The multiplicative decay factor for the learning rate. \\
\hline
\texttt{weight\_tau} & 0.9 & The multiplicative decay factor for the temperature \texttt{tau}. \\
\hline
\texttt{update\_percentage} & 0.1 & The interval for decaying LR and \texttt{tau}, as a percentage of total epochs. \\
\hline
\texttt{importance\_sampling} & True & A flag to enable importance sampling in the policy update. \\
\hline
\end{tabular}
\label{tab:NAL_hyper}
\end{table*}

\subsection{Large-Scale Game Discussions}
\label{appendix:large_discussion}

\paragraph{Evaluation in Large-Scale Games} In the large-scale games, the attacker's action space is too large to be feasibly enumerated, which makes the calculation of the duality gap impractical. To address this challenge, we adopt an alternative evaluation methodology. After the training phase for each algorithm is complete, we assess the quality of the trained policies by simulating their performance against one another. Specifically, for each matchup between an attacker policy and a defender policy, we conduct $1000$ game rollouts by sampling actions from their respective strategy distributions. The attacker's win rate is then calculated and recorded in a performance matrix, as illustrated in Figure~\ref{fig:large_exp}. The training time for PSRO is approximately two weeks, compared to four days for TSO. This difference is mainly because PSRO requires sampling a larger number of attacker actions than TSO. As sampling attacker actions is time-consuming, PSRO consequently needs a longer training period than TSO. In addition to PSRO and TSO, our comparison incorporates two baseline strategies, Uniform TSO and Uniform PSRO, which randomly choose actions from the action space and serve as benchmarks for untrained performance.

\begin{figure}[htbp]
    \centering
    \includegraphics[width=0.7\columnwidth]{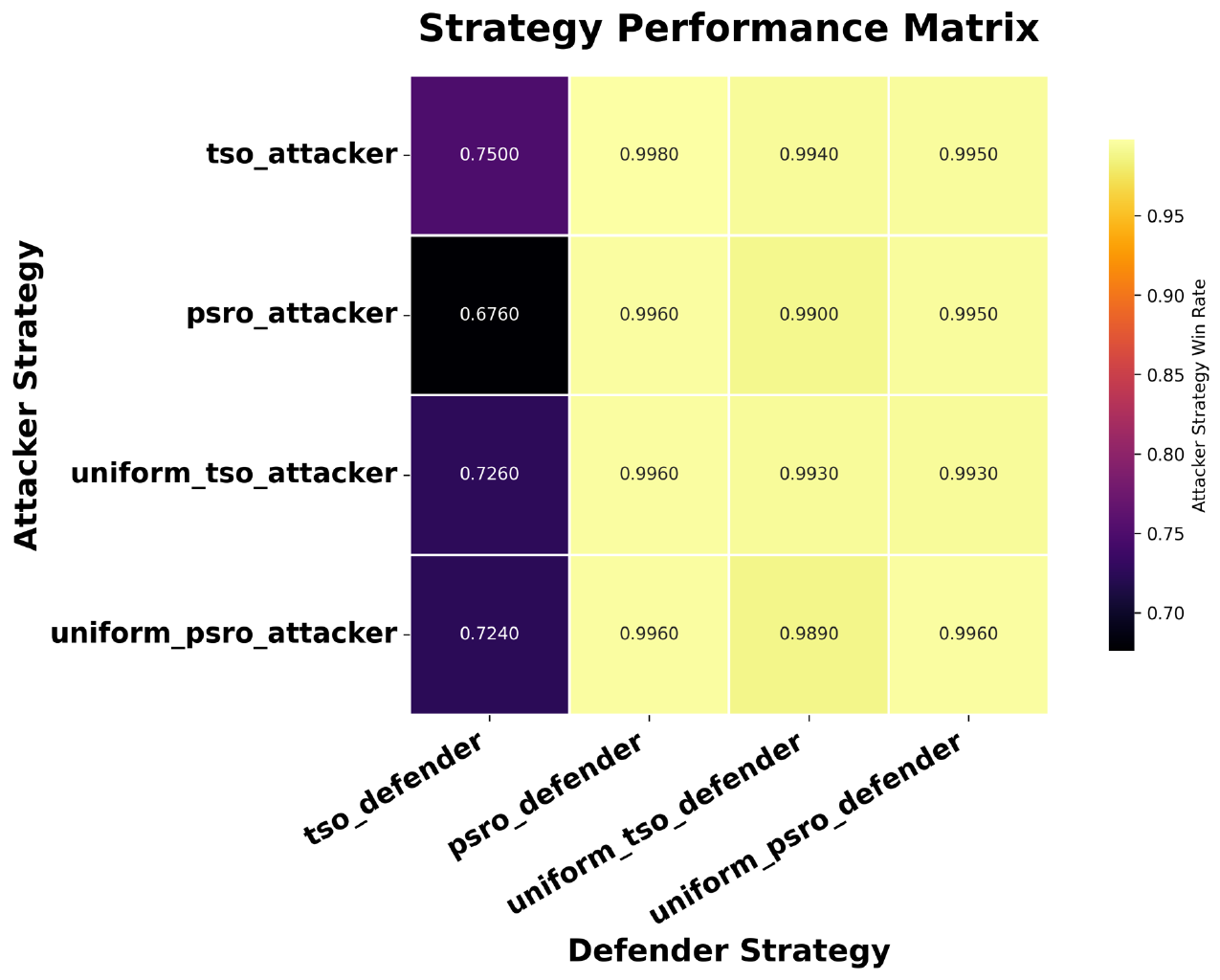}
    \caption{\small The Large-Scale Game Experiment Result}
    \label{fig:large_exp}
\end{figure}

\paragraph{Results and Analysis} The strategy performance matrix in Figure~\ref{fig:large_exp} provides an overview of the relative strengths of the policies. In this matrix:
Each row represents a fixed attacker strategy competing against different defender strategies. A lower attacker win rate within a row indicates a more effective defender policy. Each column represents a fixed defender strategy facing different attacker strategies. A higher attacker win rate within a column signifies a more capable attacker policy.

The attacker policy obtained via the TSO method demonstrates the highest strength, achieving the highest win rate in each column. The only exception is in the last column, where the TSO attacker’s win rate is just 0.1\% lower than the highest win rate in that column, where the win rate of all attacker strategies is almost 100\%. The defender policy trained using the TSO method is the strongest, as the attacker’s win rate is the lowest against the TSO defender in every row. In contrast, the PSRO framework performs similarly to the uniform baselines, i.e., the strategy learned by the PSRO is not effective. More specifically, the performance of both the psro\_attacker and psro\_defender closely matches that of their uniform counterparts.

In conclusion, the large-scale experiment validates the effectiveness of our TSO method. By producing both the strongest attacker and defender policies, TSO demonstrates a clear and significant advantage over the PSRO framework, which fails in large-scale UNSGs.

\subsection{The number of defenders Impact}
\label{appendix:defender_number}

We test the impact of the number of defenders on the performance of TSO in medium-sized UNSGs. We set up four scenarios with 2, 3, 4, and 5 defenders, respectively. Each defender can choose from 10 possible locations. The evader has 4 exit points and a maximum path length of 7. The duality gap for these scenarios are presented in Table \ref{tab:duality_gap_defenders}. As the number of defenders increases, the duality gap does not change much. This result demonstrates TSO's scalability and effectiveness in handling larger defender teams in UNSGs scenarios.

\begin{table}[ht]
\centering
\caption{Duality gap vs. defender number}
\begin{tabular}{ccccc}
\toprule
Defender number & 2 & 3 & 4 & 5 \\
\midrule
Duality gap & 0.038488 & 0.033612 & 0.031929 & 0.031548 \\
\bottomrule
\end{tabular}
\label{tab:duality_gap_defenders}
\end{table}

\subsection{Hyperparameter Analysis}
\label{sec:hyperparameter_analysis}

We conduct a grid search over three key hyperparameters in TSO: the entropy regularization coefficient \(\tau\), the update rate for \(\tau\) decay (update\_rate), and the weight for \(\tau\) decay (weight\_tau). The experiment is conducted on the \textbf{S-1} graph. The search results are summarized in Table \ref{tab:min_duality_gap_combo}.

The global minimum duality gap is 0.02822 at \( (\tau,\ \text{update\_rate},\ \text{weight\_tau})=(0.1,\ 0.1,\ 0.5) \). Among the three hyperparameters, update\_rate shows the strongest and most consistent effect. Averaged over all other settings, the mean duality gap rises as update\_rate shrinks: approximately \(0.073\) at 0.1, \(0.103\) at 0.05, and \(0.178\) at 0.025. Within each \(\tau\), the best results occur at update\_rate \(= 0.1\) (for example, at \(\tau=0.1\): 0.07584 \(\to\) 0.11268 \(\to\) 0.13214 as update\_rate decreases from 0.1 to 0.05 to 0.025).

For \(\tau\), a clear better selection emerges at the middle value. Averaged across the grid, \(\tau=0.1\) yields the lowest mean gap (\(\sim 0.098\)) compared with \(\tau=0.05\) (\(\sim 0.114\)) and \(\tau=0.2\) (\(\sim 0.142\)). This ordering also appears when holding update\_rate at 0.1 and sweeping weight\_tau: at update\_rate \(= 0.1\), the best value for each weight\_tau occurs at \(\tau=0.1\) (e.g., for weight\_tau \(= 0.5\): 0.06026 at \(\tau=0.2\), 0.02822 at \(\tau=0.1\), 0.04979 at \(\tau=0.05\)).

The effect of weight\_tau is clearly adverse at the high end: weight\_tau \(= 0.9\) is uniformly worse on average (mean \(\sim 0.151\)) than 0.7 (\(\sim 0.098\)) or 0.5 (\(\sim 0.105\)). When update\_rate is high (0.1), decreasing weight\_tau consistently helps (e.g., at \(\tau=0.1\): 0.07584 \(\to\) 0.04283 \(\to\) 0.02822 as weight\_tau goes 0.9 \(\to\) 0.7 \(\to\) 0.5). However, interactions matter when the update rate is small: at update\_rate \(= 0.025\), the optimum often shifts to weight\_tau \(= 0.7\) rather than 0.5 (e.g., \(\tau=0.05\): 0.18146 at 0.9, 0.15762 at 0.7, but a worse 0.21149 at 0.5). In short, lower weight\_tau pays off when you keep update\_rate large, while weight\_tau \(= 0.7\) is more robust if update\_rate must be reduced.

From Figure \ref{fig:batchnum_duality}, we observe that the duality gap consistently decreases as the batch\_num increases from 16 to 256. The gap drops from 0.1463 at batch\_num = 16 to 0.0481 at batch\_num = 256, indicating that larger batch sizes yield better optimization quality under the same hyperparameter setting.

Putting these pieces together, the results support the following tuning intuition: prioritize a large update\_rate (0.1) and a mid-range \(\tau\) (0.1); within that regime, a smaller weight\_decay (0.5) gives the best gaps. What's more, increasing batch\_num consistently improves performance.

\begin{figure}[ht]
\centering
\begin{minipage}[c]{0.48\textwidth}
\centering
\begin{tabular}{lc}
\toprule
$(\tau,\ \text{update\_rate},\ \text{weight\_tau})$ & Duality gap \\
\midrule
(0.2, 0.1, 0.9)   & 0.18217  \\
(0.2, 0.1, 0.7)   & 0.07320 \\
(0.2, 0.1, 0.5)   & 0.06026 \\
(0.2, 0.05, 0.9)  & 0.16915  \\
(0.2, 0.05, 0.7)  & 0.08154 \\
(0.2, 0.05, 0.5)  & 0.10309  \\
(0.2, 0.025, 0.9) & 0.30242  \\
(0.2, 0.025, 0.7) & 0.17214  \\
(0.2, 0.025, 0.5) & 0.13829  \\
(0.1, 0.1, 0.9)   & 0.07584 \\
(0.1, 0.1, 0.7)   & 0.04283 \\
(0.1, 0.1, 0.5)   & 0.02822 \\
(0.1, 0.05, 0.9)  & 0.11268  \\
(0.1, 0.05, 0.7)  & 0.08563  \\
(0.1, 0.05, 0.5)  & 0.09321 \\
(0.1, 0.025, 0.9) & 0.13214  \\
(0.1, 0.025, 0.7) & 0.12634  \\
(0.1, 0.025, 0.5) & 0.18137  \\
(0.05, 0.1, 0.9)  & 0.09387  \\
(0.05, 0.1, 0.7)  & 0.05046 \\
(0.05, 0.1, 0.5)  & 0.04979 \\
(0.05, 0.05, 0.9) & 0.11249  \\
(0.05, 0.05, 0.7) & 0.09111 \\
(0.05, 0.05, 0.5) & 0.08056 \\
(0.05, 0.025, 0.9)& 0.18146  \\
(0.05, 0.025, 0.7)& 0.15762  \\
(0.05, 0.025, 0.5)& 0.21149  \\
\bottomrule
\end{tabular}
\captionof{table}{\small Duality gap under different hyperparameter combinations.}
\label{tab:min_duality_gap_combo}
\end{minipage}%
\hfill
\begin{minipage}[c]{0.48\textwidth}
\centering
\includegraphics[width=\textwidth]{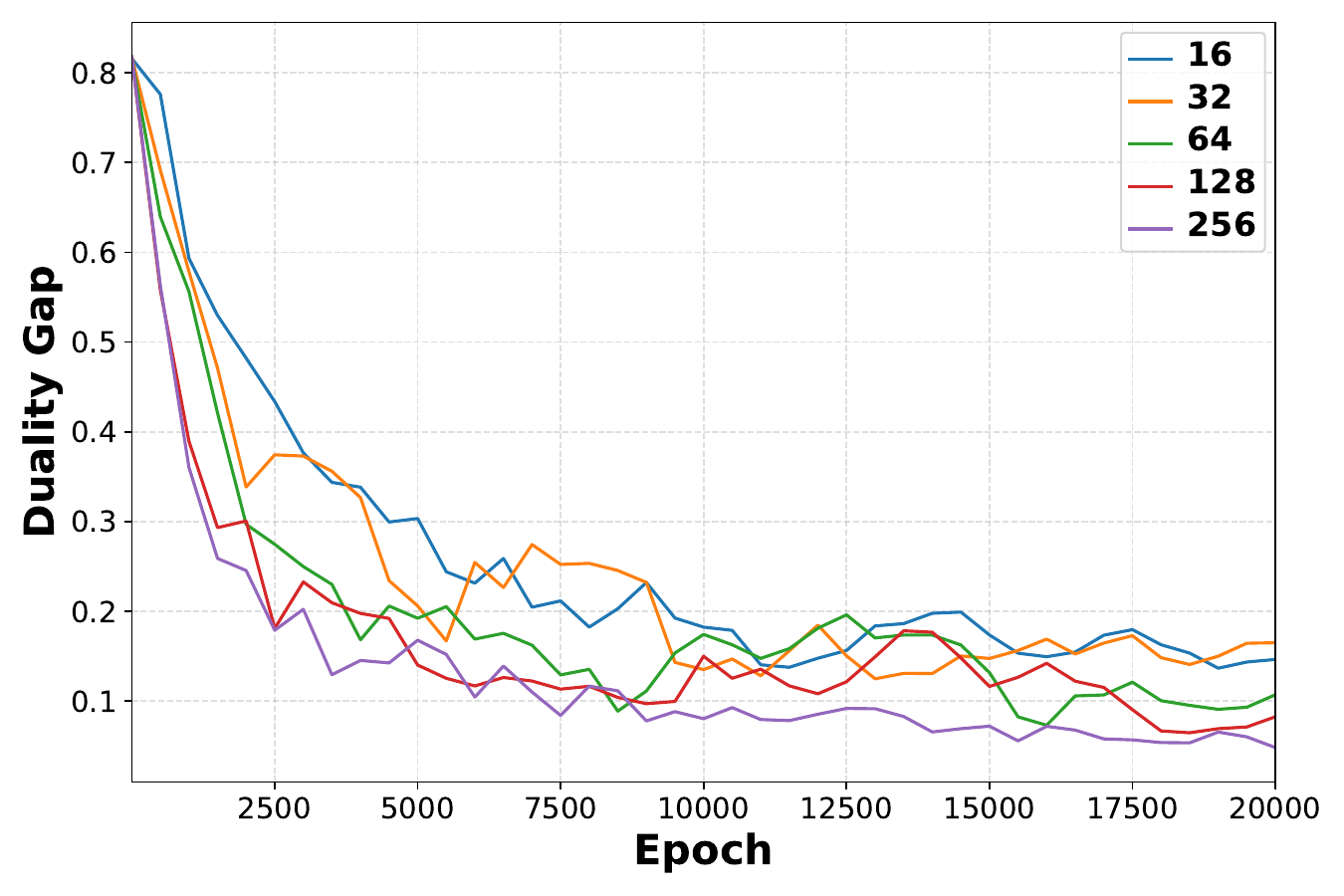}
\caption{\small Duality gap under different batch\_num with other hyperparameters kept the same.}
\label{fig:batchnum_duality}
\end{minipage}
\end{figure}

\subsection{Wall-Clock Time Analysis of TSO and PSRO}
\label{sec:time_comparison}

For both PSRO and TSO, we control the total number of sampled episodes to be the same over the entire training process. The game settings are based on \textbf{S-1}. Specifically, for PSRO we set the product of \texttt{train\_evader\_number} and the number of PSRO iterations \texttt{num\_psro\_iteration} to match the product of \texttt{batch\_num} and \texttt{total\_epoch} used in TSO, so that the total number of sampling episodes is comparable across the two methods. The results are summarized in Table \ref{tab:time_analysis}.

The row labeled \emph{training time} reports the wall-clock training for each sampling budget, where in each pair the first value corresponds to TSO and the second to PSRO. The row labeled \emph{duality gap} reports the corresponding duality gaps achieved under the same sampling budgets, again with the first value for TSO and the second for PSRO. From the table, we observe that, for the same number of sampled episodes, TSO consistently requires less wall-clock time than PSRO and achieves a substantially smaller duality gap. This indicates that TSO is more efficient than PSRO both in terms of computational cost and solution quality under matched sampling budgets.

\begin{table}[htbp]
\centering
\caption{Wall-Clock time and duality gap of TSO and PSRO under Matched Sampling Budgets}
\begin{tabular}{cccccc}
\toprule
Sample episodes & $4\times10^5$ & $8\times10^5$ & $1.2\times10^6$ & $1.6\times10^6$ & $2\times10^6$ \\
\midrule
Training time (min)   & (4, 17)           & (8, 35)           & (12, 53)           & (16, 70)          & (19, 88)         \\
Duality gap     & (0.1682, 0.7633)   & (0.1353, 0.5173)   & (0.1211, 0.4897)   & (0.0730, 0.4849)   & (0.0699, 0.4758)   \\
\bottomrule
\end{tabular}
\label{tab:time_analysis}
\end{table}







\end{document}

%% file: 3_4_details_of_TSO.tex
\subsection{Tree-Based Stochastic Optimization}
\label{subsec:Sample-and-Prune Mechanism}
In this section, we introduce the Sample-and-Prune mechanism and outline the TSO framework, showing how Tree-Based Action Representation and Sample-and-Prune are integrated to optimize our Tree-Based NAL.

\paragraph{Sample-and-Prune Mechanism.} The core of optimizing loss function with stochastic gradient descent lies in action sampling. We observe that, during the training process, a policy might discover a temporarily effective strategy and, through a strong positive feedback loop, quickly increase the probabilities of the related actions. This often results in convergence to a suboptimal local optimum. Consequently, the gradients computed from sampled actions approach zero, making it challenging for the agent to escape this suboptimal state, even when employing exploration methods such as $\epsilon$-greedy with a high exploration rate.

To address this issue, we introduce the Sample-and-Prune mechanism. This approach decomposes the action sampling process into two distinct steps, with the final action being used to compute the loss and update the network parameters. The procedure is as follows:

\begin{enumerate}
    \item \textbf{Sample:} An initial action, $a_k$, is sampled from the current policy distribution $\pi$.
    \item \textbf{Prune and Re-sample:} The initially sampled action $a_k$ is pruned from the set of available actions. A second action, $a_{k'}$, is then sampled from the same policy distribution $\pi$, conditioned on the pruned action space (i.e., $a_{k'} \neq a_k$).
\end{enumerate}

\begin{proposition}
\label{sampling_proposition}
The use of the Sample-and-Prune mechanism for action sampling does not affect the equivalence between $\mathcal{L}_\text{TSO}^\tau(\bm{x})$ and $\mathcal{L}_{\text{NAL}}^{\tau}(\bm{x})$.
\end{proposition}

Further details on the sample-and-prune mechanism and the proof of Proposition \ref{sampling_proposition} are provided in Appendix \ref{appendix:sample}.

\begin{algorithm}[t]
    \centering \small
    \caption{TSO}\label{alg:our algorithm}
    \begin{algorithmic}[1]
        \STATE {\bfseries Input:} An optimizer $\mathcal{OPT}$, the player set $\mathcal{P} = \{\text{attacker}, \text{defender}\}$, the exploration ratio $\epsilon$, the uniform strategy profile $\bm{x}^{u}=[\bm{x}^{u}_k| k \in \mathcal{P}]$, the initial parameter $\theta$, the initial parameter $\phi$, the learning rate $\eta$, the regularization scalar $\tau$, the number of total iterations $T$, the number of instances $S$ sampled at per iteration, the frequency $T_u$ of updating $\eta$ and $\tau$, the weight $\alpha$ on updating $\eta$, the weight $\beta$ on updating $\tau$, simulator $\mathcal{G}$ that returns player $i$’s payoff given a joint action.
        \FOR{each $t \in [1,\ 2,\cdots, T]$}
            \STATE $\mathcal{M}_k \leftarrow \{\}, v_k \leftarrow 0,\ \forall k \in \mathcal{P}$
            \FOR{each $s \in [1,\ 2,\cdots, S]$}
                \STATE \textbf{// Our Sample-and-Prune Mechanism}
                \STATE $a_k \sim \bm{x}^{\theta,\phi}_k,\ \forall k \in \mathcal{P}$
                \STATE $\bm{a} \leftarrow (a_k)_{k \in \mathcal{P}}$
                \STATE $\bm{x}_{k}^{\prime} \leftarrow (1 - \epsilon)\bm{x}^{\theta,\phi}_k + \epsilon \bm{x}^{u}_k,\ \forall k \in \mathcal{P}$
                \STATE $\bm{x}_{k}^{\prime}(a_k) \leftarrow 0,\ \forall k \in \mathcal{P}$
                \STATE $\bm{x}_{k}^{\prime} \leftarrow \frac{\bm{x}_{k}^{\prime}}{\Vert \bm{x}_{k}^{\prime} \Vert_1},\ \forall k \in \mathcal{P}$
                \STATE $a_{k^{\prime}} \sim \bm{x}_{k}^{\prime}, p_k \leftarrow \bm{x}_{k}^{\prime}(a_{k^{\prime}}),\ \forall k \in \mathcal{P}$
                \STATE \textbf{// To estimate $\bm{F}^{\tau, \bm{x}^{\theta,\phi}}_k(a_{k^{\prime}})$}
                \STATE $r_k \leftarrow -\mathcal{G}(k, a_{k^{\prime}}, \bm{a}_{-k})+ \tau \log \bm{x}^{\theta,\phi}_k(a_{k^{\prime}}) ,\ \forall k \in \mathcal{P}$
                \STATE $\mathcal{M}_k.\text{append}([i,\  a_{k^{\prime}},\ r_k,\ p_k]) ,\ \forall k \in \mathcal{P}$
                \STATE $v_k \leftarrow v_k + r_k$
            \ENDFOR
            \STATE $\hat{\mathcal{L}}_{\text{TSO}}^\tau(\theta,\phi) \leftarrow 0 $
            \STATE $v_k \leftarrow \frac{v_k}{S} ,\ \forall k \in \mathcal{P}$ \textbf{// To estimate $\langle \bm{F}^{\tau, \bm{x}^{\theta,\phi}}_k , \hat{\bm{x}}_k \rangle$}
            \FOR{each $k \in \mathcal{P}$}
                \FOR{each $[i, a^s_k, r^s_k, p^s_k] \in \mathcal{M}_k$}
                    \STATE \textbf{// To estimate $\bm{F}^{\tau, \bm{x}^{\theta,\phi}}_k  - \langle \bm{F}^{\tau, \bm{x}^{\theta,\phi}}_k  , \hat{\bm{x}}_k \rangle \bm{1}$}
                    \STATE $\bm{g}^s_k \leftarrow \frac{r^s_k - v_k }{p^s_k} \bm{e}_{a^s_k}$
                    \STATE $\hat{\mathcal{L}}_{\text{TSO}}^\tau(\theta,\phi) \leftarrow \hat{\mathcal{L}}_{\text{TSO}}^\tau(\theta,\phi) + \langle sg[\bm{g}^s_k], {\bm{x}}^{\theta,\phi}_k\rangle$
                \ENDFOR
            \ENDFOR
            \STATE $\theta,\phi \leftarrow \mathcal{OPT}.\text{update}(\hat{\mathcal{L}}_{\text{TSO}}^\tau(\theta,\phi))$
            \IF{$t\%T_u=0$}
                \STATE $\eta \leftarrow \alpha \eta$, $\tau \leftarrow \beta \tau$
            \ENDIF
        \ENDFOR
        \STATE {\bfseries Return} $\theta,\phi$
    \end{algorithmic}
\end{algorithm}

\paragraph{Outline of TSO.} Next, we present the pseudocode of TSO in \Cref{alg:our algorithm}, illustrating the integration of Tree-Based Action Representation and the Sample-and-Prune Mechanism (Line 5-11) into stochastic optimization for optimizing the Tree-Based NAL.
For clarity, we adopt the normal-form game strategy representation, denoting the strategy profile parameterized by $\theta$ and $\phi$ as $\bm{x}^{\theta,\phi} = (\bm{x}^{\theta,\phi}_{\text{attacker}}, \bm{x}^{\theta,\phi}_{\text{defender}})$, while omitting the details of our Tree-Based Action Representation.

Specifically, at each iteration $t$, for the $s$-th sample, the first step of our Sample-and-Prune Mechanism—the ``Sample" phase—is to independently sample, for each player, an action $a_k$ from the strategy profile $\bm{x}^{\theta,\phi}$ using our Tree-Based Action Representation (Line 6). This sampled $a_k$ not only is used in our Sample-and-Prune Mechanism but also serves as an unbiased estimator for environmental dynamics, enabling unbiased estimation of $\mathcal{L}_{\text{TSO}}^{\tau}$. Formally, this sampled $a_k$ for each $k \in \mathcal{P}$ yields the action profile $\bm{a} \leftarrow (a_k)_{k \in \mathcal{P}}$ (Line 7). For any player $k$, $\bm{a}_{-k} = (a_j)_{j \in \mathcal{P},\, j \neq k}$ serves as an unbiased estimate for environmental dynamics.

Subsequently, the second step of our Sample-and-Prune Mechanism—the ``Prune and Re-sample" phase—is, for each player $k$, an alternative action $a_{k^\prime} \neq a_k$ is drawn the following process. First, a modified strategy $\bm{x}_{k}^{\prime} \leftarrow (1 - \epsilon)\bm{x}^{\theta,\phi}_k + \epsilon \bm{x}^{u}_k$ is constructed for each $k \in \mathcal{P}$ (Line 8). Then, the probability $\bm{x}_{k}^{\prime}(a_k)$ is set to zero and $\bm{x}_{k}^{\prime}$ is renormalized to maintain it within simplex (Line 9-10). Finally, an alternative action $a_{k^\prime} \neq a_k$ is sampled from $\bm{x}_{k}^{\prime}$ (Line 11).

Due to space constraints, we do not further elaborate on the remaining components of \Cref{alg:our algorithm}. Detailed explanations are available in Appendix \ref{sec:Overview of TSO}.

{

}

%% file: 6_appendxix_outline.tex
\section{Overview of TSO}
\label{sec:Overview of TSO}

We now provide the overview of our TSO, as demonstrated in \Cref{alg:our algorithm}. The goal of TSO is to minimize the following loss function:
\begin{equation}
\setlength\abovedisplayskip{2pt}
\setlength\belowdisplayskip{2pt}
\thinmuskip=0mu
\medmuskip=0mu
\thickmuskip=0mu
\spaceskip=-0pt 
    \begin{aligned}
        \mathcal{L}_{\text{TSO}}^\tau(\theta,\phi)
        &= \sum_{k \in \mathcal{P}}\left\langle \text{sg}\left[\bm{F}_k^{\tau, \bm{x}^{\theta,\phi}} - \langle \bm{F}_k^{\tau, \bm{x}^{\theta,\phi}}, \hat{\bm{x}}_k \rangle \mathbf{1}\right], \bm{x}^{\theta,\phi}_k \right\rangle.
    \end{aligned}
\end{equation}
To achieve this goal, TSO comprises two main steps: sampling and updating.

\textbf{Step 1.} The sampling phase, described in lines 3–16, \Cref{alg:our algorithm}. At each iteration $t$, we initialize the buffer $\mathcal{M}_k$ and the variable $v_k$, which are used to estimate the value of $ \mathcal{L}_{\text{TSO}}^\tau(\theta,\phi)$. Subsequently, for each player $k \in \mathcal{P}$, we generate $S$ samples. In each sample, an action $a_k$ is drawn for each player based on the strategy profile $\bm{x}^{\theta,\phi}$, utilizing our Tree-Based Action Representation (line 6, \Cref{alg:our algorithm}). This yields the action profile $\bm{a} \leftarrow (a_k)_{k \in \mathcal{P}}$ (line 7, \Cref{alg:our algorithm}). For any player $k$, $\bm{a}_{-k} = (a_j)_{j \in \mathcal{P},\, j \neq k}$ serves as an unbiased estimate for environmental dynamics. Next, as mentioned in our Tree-Based Action Representation, by employing our Tree-Based Action Representation, we sample for each player $k$ an alternative action $a_{k^{\prime}} \neq a_k$, along with its sampling probability $p_k$. The resulting unbiased estimate of $\bm{F}^{\tau, \bm{x}^{\theta,\phi}}_k(a_{k^{\prime}})$ is given by $r_k = -\mathcal{G}(k, a_{k^{\prime}}, {a}_{-k})+ \tau \log \bm{x}^{\theta,\phi}_k(a_{k^{\prime}})$ for all $k \in \mathcal{P}$, where $\mathcal{G}$ is the simulator that returns the payoff to player $k$ for the joint action $(a_{k^{\prime}}, \bm{a}_{-k})$. Finally, the tuple $[k,\ a^{\prime}_k,\ r_k,\ p_k]$ is stored in the buffer $\mathcal{M}_k$ (line 14, \Cref{alg:our algorithm}), and $v_k$ is updated according to $v_k \leftarrow v_k + r_k$ (line 15, \Cref{alg:our algorithm}).

\textbf{Step 2.} The updating phase, described in lines 17–29, \Cref{alg:our algorithm}. We first initialize the estimator for $\mathcal{L}_{\text{TSO}}^\tau(\theta,\phi)$ as $\hat{\mathcal{L}}_{\text{TSO}}^\tau(\theta,\phi) \leftarrow 0$ and normalize $v_i$ by setting $v_i \leftarrow \frac{v_i}{S}$ (lines 17-18, \Cref{alg:our algorithm}). Immediately, we get that the expectation $\mathbb{E}[v_i]$ corresponds to $ \langle \bm{F}^{\tau, \bm{x}^{\theta,\phi}}_i , \hat{\bm{x}}_i \rangle$. Additionally, we use the tuples in $\mathcal{M}_i$ (line 20, \Cref{alg:our algorithm}) to estimate $\bm{F}^{\tau, \bm{x}^{\theta,\phi}}_k  - \langle \bm{F}^{\tau, \bm{x}^{\theta,\phi}}_k  , \hat{\bm{x}}_k \rangle \bm{1}$ through the computation $\bm{g}^s_k \leftarrow \frac{r^s_k - v_k }{p^s_k} \bm{e}_{a^s_k}$ (line 22, \Cref{alg:our algorithm}), where $\bm{e}_{a^s_i}$ is a vector whose the coordinate $a^s_i$ is $1$ and all other coordinates are $0$. Obviously, we have that $\mathbb{E}[\bm{g}^s_k] = \bm{F}^{\tau, \bm{x}^{\theta,\phi}}_k  - \langle \bm{F}^{\tau, \bm{x}^{\theta,\phi}}_k  , \hat{\bm{x}}_k \rangle \bm{1}$. The estimator {\small$\hat{\mathcal{L}}_{\text{TSO}}^\tau(\theta,\phi)$} is updated via {\small$\hat{\mathcal{L}}_{\text{TSO}}^\tau(\theta,\phi) \leftarrow \hat{\mathcal{L}}_{\text{TSO}}^\tau(\theta,\phi) + \langle \bm{g}^s_i, {\bm{x}}^{\theta,\phi}_i \rangle$} (line 23, \Cref{alg:our algorithm}). Since $\mathbb{E}[\bm{g}^s_k] = \bm{F}^{\tau, \bm{x}^{\theta,\phi}}_i  - \langle \bm{F}^{\tau, \bm{x}^{\theta,\phi}}_k  , \hat{\bm{x}}_k \rangle \bm{1}$ and ${\bm{x}}^{\theta,\phi}_k$ is known, it follows that $\frac{1}{S}\mathbb{E}[\hat{\mathcal{L}}_{\text{TSO}}^\tau(\theta,\phi)] = \mathcal{L}^{\tau}_{NAL}(\theta,\phi)$. Therefore, $\hat{\mathcal{L}}_{\text{TSO}}^\tau(\theta,\phi)$ provides an unbiased estimate of $\mathcal{L}^{\tau}_{NAL}(\theta,\phi)$. The estimator $\hat{\mathcal{L}}_{\text{TSO}}^\tau(\theta,\phi)$ is then passed to the optimizer $\mathcal{OPT}$ for updating $\theta\ \text{and}\ \phi$ (line 26, \Cref{alg:our algorithm}). Lastly, as in \citet{mengreducing}, if $t \% T_u = 0$ (line 27, \Cref{alg:our algorithm}), the parameters $\eta$ and $\tau$ are updated as $\eta \leftarrow \alpha \eta$ and $\tau \leftarrow \beta \tau$ (line 27, \Cref{alg:our algorithm}) to enhance the stability of TSO, where $0 < \alpha, \beta < 1$